\pdfoutput=1
\PassOptionsToPackage{table}{xcolor}
\documentclass[11pt]{article}

\usepackage{acl}

\usepackage{times}
\usepackage{latexsym}

\usepackage[T1]{fontenc}

\usepackage[utf8]{inputenc}

\usepackage{microtype}

\usepackage{inconsolata}

\usepackage{graphicx}
\usepackage{fontawesome5}
\usepackage[misc]{ifsym} 

\usepackage{booktabs}
\usepackage{amsmath}
\usepackage{amssymb}
\usepackage{enumitem} 
\usepackage{subcaption} 
\usepackage{multirow}
\usepackage[capitalize,noabbrev]{cleveref}
\usepackage{algorithm}
\usepackage{algpseudocode}

\newtheorem{theorem}{Theorem}[section]

\newtheorem{proof}{Proof}
\newtheorem{corollary}[theorem]{Corollary}

\newtheorem{assumption}[theorem]{Assumption}

\usepackage[breakable]{tcolorbox}
\DeclareRobustCommand{\remarkbox}[2][blue!6]{%
\begin{tcolorbox}[
        breakable,
        left=0pt,
        right=0pt,
        top=0pt,
        bottom=0pt,
        colback=#1,
        colframe=blue!6,
        width=\columnwidth,
        arc=0pt,outer arc=0pt,
        ]
        #2
\end{tcolorbox}
}

\newcommand{\std}[1]{\raisebox{-0.4ex}{\hspace{0.5pt}{\tiny\color{gray}$\pm$#1}}}

\title{Localize-Then-Decide Guarantees for LLM Judgments}

\author{
  \textbf{Xinyu Li\textsuperscript{1}},
  \textbf{Yi Zhou\textsuperscript{2}},
  \textbf{Guanqun Cao\textsuperscript{3}},
  \textbf{Zeyu Fu\textsuperscript{1}},
  \textbf{Tianjin Huang\textsuperscript{1}},
  \textbf{Gaojie Jin\textsuperscript{4,\Letter}}
\\
  \textsuperscript{1}University of Exeter\quad
  \textsuperscript{2}Cardiff University\quad
  \textsuperscript{3}University of the West of England\\
  \textsuperscript{4}University of Macau
\\
  Correspondence to: Gaojie Jin \texttt{<gaojie.jin.kim@gmail.com>}
}

\begin{document}
\maketitle
\begin{abstract}
Large language models (LLMs) are increasingly used as evaluators to assess output quality and preference alignment, yet providing reliable guarantees of agreement with human judgments remains challenging. Recent work introduces confidence-thresholding methods that provide such guarantees for pairwise comparisons, relying on the assumption that higher estimated confidence implies lower disagreement risk with humans. However, this assumption can break down when the number of candidate responses increases, since distributing probability mass across many alternatives can distort confidence estimates. To address this issue, we propose a Localize-Then-Decide framework. First, conformal prediction localizes a small shortlist that contains the human-preferred response with high probability. Then, a calibrated confidence-based rule selectively chooses a single response from this shortlist or abstains. This design restores the monotonic relationship between confidence and disagreement risk and enables high-probability agreement guarantees.
Experiments with multiple candidate sizes across several datasets and judge LLMs demonstrate that our framework consistently achieves higher guarantee success rates and substantially higher coverage than single-stage baselines.~\href{https://github.com/llm2409/Localize-Then-Decide}{\faGithub}
\end{abstract}

\section{Introduction}
\label{sec:intro}

Large language models (LLMs) are increasingly deployed as evaluators for assessing output quality and preference alignment~\citep{zheng2023judging,dubois2023alpacafarm,park2025adaptive,chiang2023can}.
Although this paradigm enables scalable and cost-effective alternatives to human annotation, a fundamental challenge persists~\citep{xiong2023can}: 
how can decisions made by LLMs acting as judges be made reliably trustworthy for downstream use, particularly when the model expresses high confidence?
To address this issue, \citet{jung2024trust} derive exact disagreement-risk guarantees conditioned on a calibration set, but restrict attention to the pairwise generations ($m=2$).
In practical deployments, however, a judge LLM is typically required to select a single human-preferred response from multiple candidates ($m>2$).
Extending such reliability guarantees beyond the pairwise regime is therefore essential for real-world applicability.

\begin{figure}
    \centering
    \includegraphics[width=1\linewidth]{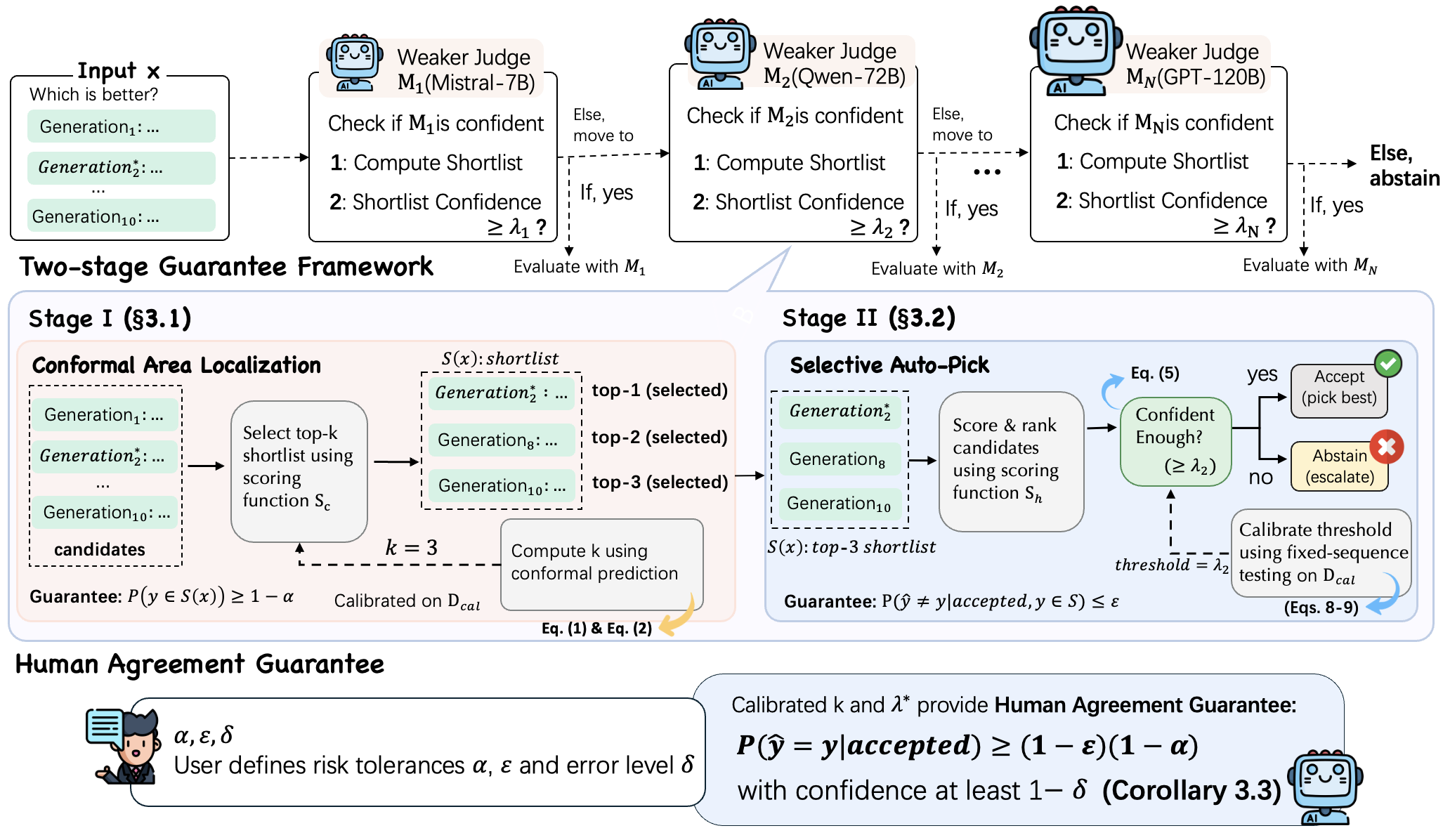}
    \vspace{-8mm}
    \caption{An overview of the Localize-Then-Decide (two-stage) guarantees framework.}
    \label{fig:framework}
\vspace{-5mm}
\end{figure}

\begin{figure*}[t]
    \centering
    \includegraphics[width=\textwidth]{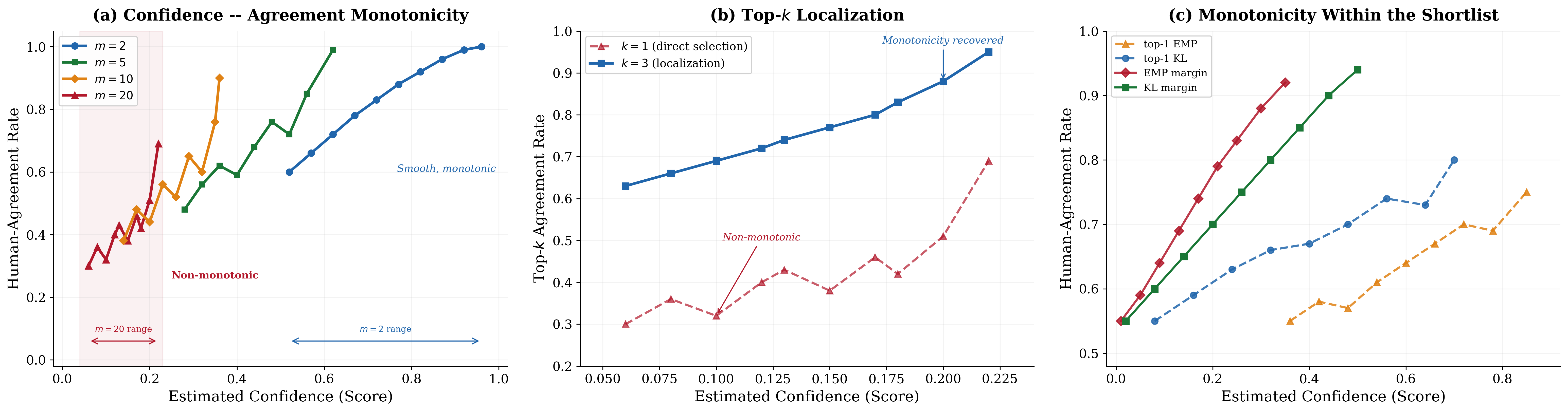}
    \vspace{-8mm}
    \caption{\textbf{Simulations on confidence--agreement monotonicity.}
    \textbf{(a)}~As the number of candidates $m$ increases, the monotone relationship between confidence scores and (above confidence) human-agreement rate progressively degrades. 
    \textbf{(b)}~Fixing $m{=}20$: although top-$1$ agreement loses monotonicity, top-$3$ agreement remains monotone with the same estimated confidence. 
    \textbf{(c)}~Within the localized (top-$3$) shortlist, the four estimated confidences (especially the margin ones) restore monotonicity. 
    Details of the simulations are given in \Cref{App:simulation}.
    }
    \label{fig:motivation}
\vspace{-5mm}
\end{figure*}

A key assumption underlying the above confidence-thresholding guarantees~\citep{jung2024trust} is monotonicity: instances assigned higher estimated confidence (probability) should exhibit lower disagreement risk with respect to human judgments.
However, this assumption becomes problematic in the multi-candidate setting. 
When an LLM judge assigns preference probabilities over $m$ candidates, the total probability mass must be distributed across all $m$ options. As $m$ increases, and the task typically becomes more challenging, the probability allocated to each individual candidate, including the true best, necessarily decreases as a direct consequence of normalization.
This structural dilution effect can violate the monotonicity assumption and thereby complicates the direct extension of pairwise judgment guarantees to the multi-candidate regime.
As shown in \Cref{fig:motivation}(a), increasing $m$ breaks the monotone confidence–agreement relationship; 
our experiments in \Cref{sec:experiments} further corroborate this phenomenon.

Although large $m$ may disrupt the monotonicity of top-$1$ agreement with respect to confidence scores, we empirically observe that top-$k$ agreement (e.g., top-3) and subsequent $k$-to-$1$ selection (e.g., selecting one from three) preserve a substantially more stable monotonic relationship with confidence measures.
As shown in \Cref{fig:motivation}(b), even when $m=20$, the top-3 agreement rate remains monotone with respect to the estimated confidence. 
Furthermore, as shown in \Cref{fig:motivation}(c), once the candidate set is restricted to the top-3 responses, the confidence–agreement relationship for selecting one out of three candidates becomes significantly more monotone (especially the margin-based confidence scores).
These observations motivate a two-stage guarantee framework for aligning LLM judgments with human preferences. 
In the first stage, we employ conformal prediction to localize a small candidate set (top-$k$) that contains the human-preferred answer with high probability. 
In the second stage, we attempt to automatically select a single response from this localized set while providing a high-probability correctness guarantee; otherwise, the system abstains.
We now present an informal summary of the resulting guarantee.
\remarkbox{
\begin{theorem}[Informal]
In a multi-candidate setting, for user-specified risk tolerances $\alpha$ (Stage~I: conformal area localization) and $\varepsilon$ (Stage~II: selective auto-picking), with confidence at least $1-\delta$, the proposed framework guarantees
\begin{equation}\nonumber
\begin{aligned}
\mathbb{P}(\text{LLM agrees with human} &| \text{Satisfied confidence})\\
&\quad\ge(1-\varepsilon)(1-\alpha).
\end{aligned}
\end{equation}
\end{theorem}
}

Empirically, we demonstrate that, under the proposed framework, the scoring (confidence) functions exhibit a substantially improved monotonic relationship between confidence and disagreement risk, leading to higher success rates in achieving target human–LLM agreement levels across multiple datasets and judge LLMs.
To summarize, the main contributions of this work are as follows:
\begin{itemize}[leftmargin=*]
\item[$\star$] \textbf{Practical Extension.} 
We extend existing human–LLM judgment guarantee frameworks from the pairwise setting ($m=2$) to the more practical multi-generation setting ($m>2$). We show that the presence of multiple candidates can cause estimated confidence scores to violate the monotonicity assumption that underlies confidence-thresholding guarantees.

\item[$\star$] \textbf{Two-Stage Guarantee.} 
To address this issue, we propose a two-stage guarantee framework for aligning LLM judgments with human preferences, motivated by the observation that such a decomposition restores the monotonic relationship between confidence and agreement. In the first stage, we employ conformal prediction to localize a small candidate set that contains the human-preferred response with high probability. In the second stage, we automatically select a single response from this localized set while providing a high-probability correctness guarantee; otherwise, the system abstains or escalates.

\item[$\star$] \textbf{Experiments.} 
We conduct extensive experiments across multiple datasets (e.g., TL;DR, Chatbot Arena, HH-RLHF, and AlpacaEval) and judge LLMs (e.g., Llama, Qwen, GPT-OSS, and DeepSeek) under various candidate-count settings ($m\in\{5,10,20\}$). 
The results demonstrate that our framework effectively restores the monotonic relationship between confidence and disagreement risk and significantly improves the reliability of human–LLM agreement guarantees.
\end{itemize}

\section{Preliminary}

\subsection{Problem Setup}

Let $f_{\mathrm{LM}}:\mathcal{X}\rightarrow\mathcal{Y}$ denote an LLM-based judge.  
Each input $x\in\mathcal{X}$ consists of a query together with $m$ candidate responses $\{g_1,\ldots,g_m\}$, and the true label $y\in\mathcal{Y}=\{1,\ldots,m\}$ indicates the index of the response preferred by the human.  
We write $z=(x,y)$ for a labeled instance drawn from an unknown underlying distribution.

Given $f_{\mathrm{LM}}$, we consider two real-valued scoring (confidence) functions:  
$\mathbb{S}_c(x,g_i)$, used in the first-stage conformal prediction procedure, and  
$\mathbb{S}_h(x,g_i)$, used in the second-stage hypothesis testing procedure.  
In both cases, larger scores indicate higher predicted preference for $g_i$, i.e., a greater likelihood of being the human-best response.
Our objective is to output:
\begin{enumerate}[leftmargin=*]
\item A small area (shortlist) $\mathcal{S}(x)\subseteq\{1,\ldots,m\}$ that contains the true label $y$ with high probability;   
\item Optionally, a single automatically selected index $\hat y\in\mathcal{S}(x)$ when such a choice can be certified as reliable. Otherwise, the system abstains or escalates the decision.
\end{enumerate}

\subsection{Assumptions}

Our framework relies on the following standard and widely adopted assumptions.

\begin{assumption}
\label{ass:exch-c}
Let $\mathcal{D}=\{z_i=(x_i,y_i)\}_{i=1}^n$ be a calibration set and $z_{n+1}$ a future test point.
Assume $(z_1,\dots,z_n,z_{n+1})$ are exchangeable (i.e., i.i.d.).
\end{assumption}

\begin{assumption}
\label{ass:unique-c}
For every $x$ and the corresponding $\{g_1,..., g_m\}$, the human-best preference $y$ is unique.
Ties in ranking by scoring functions are broken deterministically.
\end{assumption}

\section{Two-Stage Certified Selection for Multi-Generations}
\label{sec:twofold-coupled}

\subsection{Stage~I: Conformal Area Localization}
\label{sec:stage1-c}

In this subsection, we describe a conformal prediction procedure for localizing a small area (shortlist) $\mathcal{S}(x)$ that contains the true label $y$ with probability at least $1-\alpha$.
We first define a nonconformity score based on the rank of the human-best response under the scoring function $\mathbb{S}_c(x,\cdot)$:
\begin{small}
\begin{equation}
R(z)
\;=\;
\big|\{ i\in\{1,\dots,m\}: \mathbb{S}_c(x,g_i) \ge \mathbb{S}_c(x,g_y) \}\big|,
\label{eq:rank-nonconf-c}
\end{equation}
\end{small}which corresponds to the number of candidates whose score is no smaller than that of the human-best response.
Given the calibration set, let $k$ denote the conformal quantile defined as
\begin{small}
\begin{equation}
k=
\text{the } \left\lceil (n+1)(1-\alpha)\right\rceil\text{-th order statistic of } \{R(z_i)\}_{i=1}^{n}.
\label{eq:k-def-c}
\end{equation}
\end{small}For a new test input $x$, we construct the area set
\begin{small}
\begin{equation}
\mathcal{S}(x) = \{\text{indices of the top-$k$ candidates under } \mathbb{S}_c(x,\cdot)\}.
\label{eq:area-set-c}
\end{equation}
\end{small}We now state the finite-sample coverage guarantee.
\remarkbox{
\begin{theorem}[Finite-sample area guarantee]
\label{thm:area-c}
Under Assumptions~\ref{ass:exch-c}, \ref{ass:unique-c}, for any test input $x$, the area set $\mathcal{S}(x)$ constructed in \eqref{eq:area-set-c} satisfies
\begin{equation}
\mathbb{P}\big(y\in \mathcal{S}(x)\big) \;\ge\; 1-\alpha.
\label{eq:area-guarantee-c}
\end{equation}
\end{theorem}
}
\noindent
\emph{Proof.}
\emph{See \Cref{App:proof}.}
\hfill $\square$

\subsection{Stage~II: Selective Auto-Pick}
\label{sec:stage2-c}

In this subsection, we seek to automatically select a single index from the localized area $\mathcal{S}(x)$ while providing a high-probability correctness guarantee; otherwise, the system abstains or escalates.

Let $\hat y$ denote the index predicted by the judge LLM $f_{\mathrm{LM}}$.
We define a margin-based confidence score within $\mathcal{S}(x)$ as
\begin{equation}
\Delta_{\mathcal{S}}(x)
\;=\;
\mathbb{S}_h(x,g_{_{\hat y}})
\;-\;
\max_{i\in \mathcal{S}(x),\, i\neq \hat y}
\mathbb{S}_h(x,g_i).
\label{eq:marginS}
\end{equation}

Given a threshold $\lambda$, we define the conditional error rate of the auto-pick rule restricted to $\mathcal{S}(x)$ as
\begin{equation}
R_{\mathcal{S}}(\lambda)
\;=\;
\mathbb{P}\!\left(\hat y\neq y \;\middle|\; \Delta_{\mathcal{S}}(x)\ge \lambda\right).
\label{eq:selective-riskS}
\end{equation}
Within $\mathcal{S}(x)$, observe that
$\{\hat y \neq y\}
\;\subseteq\;
\{y\notin \mathcal{S}(x)\}
\;\cup\;
\{\hat y \neq y,\; y\in \mathcal{S}(x)\}$.
Therefore, $\forall\; \lambda$,
\begin{equation}
\begin{aligned}
&R_{\mathcal{S}}(\lambda)
\;\le\;
\underbrace{\mathbb{P}\!\left(y\notin \mathcal{S}(x) \;\middle|\; \Delta_{\mathcal{S}}(x)\ge \lambda\right)}_{\text{localization miss under acceptance}}\\
&\quad\quad\;+
\underbrace{\mathbb{P}\!\left(\hat y\neq y \;\middle|\; \Delta_{\mathcal{S}}(x)\ge \lambda,\; y\in \mathcal{S}(x)\right)}_{\text{within-area selection error}} .
\end{aligned}
\label{eq:RS-decomp}
\end{equation}
In \eqref{eq:RS-decomp}, the second term corresponds to the component that Stage~II can directly calibrate, whereas the first term is partly controlled by the Stage~I localization guarantee.

To calibrate the within-area selection error, we still use the calibration set
$\mathcal{D}=\{z_j\}_{j=1}^{n}$.
For a given threshold $\lambda$, define the set of accepted-and-covered indices
\begin{equation}\nonumber
\begin{aligned}
&A^{\mathrm{in}}(\lambda)
\;=\;
\{ j : \Delta_\mathcal{S}(x_j)\ge \lambda \ \text{and}\ y_j\in \mathcal{S}(x_j) \},\\
&N^{\mathrm{in}}(\lambda)=|A^{\mathrm{in}}(\lambda)|,
\end{aligned}
\end{equation}
and the corresponding number of within-area errors
\[
Y^{\mathrm{in}}(\lambda)=\sum_{j\in A^{\mathrm{in}}(\lambda)} \mathbf{1}\{\hat y_j\ne y_j\}.
\]
Whenever $N^{\mathrm{in}}(\lambda)\ge 1$, we define the empirical and population within-area error rates as
\begin{equation}\nonumber
\begin{aligned}
&\widehat R_\mathcal{S}^{\mathrm{in}}(\lambda)
\;=\;
\frac{Y^{\mathrm{in}}(\lambda)}{N^{\mathrm{in}}(\lambda)}, \\
&R_\mathcal{S}^{\mathrm{in}}(\lambda)
\;=\;\mathbb{P}\!\left(\hat y\neq y \;\middle|\; \Delta_{\mathcal{S}}(x)\ge \lambda,\; y\in \mathcal{S}(x)\right).
\end{aligned}
\end{equation}
Since $Y^{\mathrm{in}}(\lambda)\mid N^{\mathrm{in}}(\lambda)\sim \mathrm{Bin} \big(N^{\mathrm{in}}(\lambda), R_\mathcal{S}^{\mathrm{in}}(\lambda)\big)$
under exchangeability,
we define the exact $(1-\delta)$ upper confidence bound (UCB):
\begin{equation}
\label{eq:Rin-ucb-fixedseq}
\begin{aligned}
\widehat R_{\mathcal{S}}^{\mathrm{in},+}(\lambda)
\;:=\;
\sup\Big\{ r\in[0,1] &:
\mathbb{P}\big(\mathrm{Bin}(N^{\mathrm{in}}(\lambda), r)\\
&\le Y^{\mathrm{in}}(\lambda)\big)\ \ge\ \delta \Big\}.
\end{aligned}
\end{equation}

\paragraph{Fixed-sequence threshold selection.}
Empirically, decreasing $\lambda$ tends to near-monotonically increase risk (see Figures~\ref{fig:motivation}(c), \ref{fig:real-motivation}(c)), because more
ambiguous instances are accepted; i.e., $\widehat R_{\mathcal{S}}^{\mathrm{in},+}(\lambda)$ tends to increase as $\lambda$ decreases. 
This allows us to use fixed sequence testing \citep{bauer1991multiple}, wherein we test from the largest value of $\lambda$ (e.g., 0.99) to
a progressively smaller value, and stop at the last time $\widehat R_{\mathcal{S}}^{\mathrm{in},+}$ is below the target risk $\varepsilon$, i.e.,
\begin{small}
\begin{equation}
\lambda^*
\;=\;
\inf\Big\{ \lambda:
\widehat R_{\mathcal{S}}^{\mathrm{in},+}(\lambda') \le \varepsilon\ \ \text{for all}\ \ \lambda'\ge \lambda
\Big\}.
\label{eq:lambda-star-fixedseq}
\end{equation}
\end{small}This is precisely the fixed-sequence testing rule adopted for selective evaluation to avoid Bonferroni
conservatism in large hypothesis spaces.
Then, we get the following guarantee.

\remarkbox{
\begin{theorem}
\label{thm:within-area-fixedseq}
Given Assumptions~\ref{ass:exch-c}, \ref{ass:unique-c}, 
let $\lambda^*$ be selected by the fixed-sequence rule~\eqref{eq:lambda-star-fixedseq}
using the binomial upper confidence bound defined in~\eqref{eq:Rin-ucb-fixedseq}, applied to the calibration subset
\[
\mathcal D_{\mathrm{in}}
:= \left\{ (x_i,y_i)\in\mathcal D \;:\; y_i \in \mathcal{S}(x_i) \right\}.
\]
Then, with probability at least $1-\delta$ over the draw of the calibration sample,
the following holds for a new example $(x,y)$:
\begin{equation}
\mathbb{P}\!\left(\hat y\ne y\ \middle|\ \Delta_\mathcal{S}(x)\ge \lambda^*, y\in\mathcal{S}(x)\right)
\;\le\;
\varepsilon.
\label{eq:within-area-fixedseq-guarantee}
\end{equation}
\end{theorem}}

\noindent\emph{Proof.}
\emph{See \Cref{App:proof}.}
\hfill $\square$

Empirically, increasing $\lambda$ approximately increases the top-$k$ agreement rate (i.e., decreases $\mathbb P(y\notin\mathcal{S}(x)\mid\Delta_\mathcal{S}(x)\ge\lambda)$), as directly tested in \Cref{tab:cross-stage-monotonicity}.
Under this empirically supported monotonicity assumption, we derive a composed bound for the selective error of the auto-pick decision.
Stage~I's marginal coverage guarantee is distribution-free under exchangeability, whereas the composed end-to-end guarantee below additionally relies on this cross-stage monotonicity condition.

\remarkbox{
\begin{corollary}
\label{thm:composed-c}
Given Assumptions~\ref{ass:exch-c}, \ref{ass:unique-c}, and the monotonicity assumption on
$\mathbb P(y\notin \mathcal{S}(x)\mid \Delta_\mathcal{S}(x)\ge \lambda)$,
let $\mathcal{S}(x)$ be produced by Stage~I and $\lambda^*$ calibrated by Stage~II as in
\Cref{thm:within-area-fixedseq}.
Then, with probability at least $1-\delta$ over the calibration sample used in
Stage~II, we get
\begin{equation}\nonumber
\mathbb{P}\big(\hat y= y\ \big|\ \Delta_\mathcal{S}(x)\ge \lambda^*\big)
\;\ge\;
(1-\varepsilon)(1-\alpha).
\label{eq:composed-bound}
\end{equation}
\end{corollary}}

\noindent\emph{Proof.}
\emph{See \Cref{App:proof}}.
\hfill $\square$

\remarkbox{
\textbf{Remark.}
\emph{
Based on the monotonicity of top-$k$ agreement (e.g., top-3) and the subsequent $k$-to-$1$ selection, we combine conformal prediction with fixed-sequence testing to obtain Corollary~\ref{thm:composed-c}. 
The result yields an interpretable decomposition of the final selective accuracy guarantee into two factors: $(1-\alpha)$, the coverage probability of Stage~I localization (i.e., the human-preferred response lies in the shortlist $\mathcal{S}(x)$), and $(1-\varepsilon)$, the conditional correctness of the Stage~II decision rule on accepted instances. 
This multiplicative form highlights the role of the two-stage design: Stage~I localizes reliable candidates, while Stage~II controls decision risk through calibrated thresholding. 
Improving either shortlist coverage or conditional decision reliability directly strengthens the overall guarantee.
}}

\section{Experiments}
\label{sec:experiments}

The previous sections establish the framework that decomposing $m$-to-$1$ selection into two stages to restore the monotonic confidence--agreement structure needed for valid guarantees.
We now evaluate whether this theoretical benefit materializes on real-world benchmarks.
Our experiments address three questions:
\textbf{(Q1)}~Does the monotonicity breakdown occur on real data, and does the two-stage decomposition recover it (\Cref{sec:exp-motivation})?
\textbf{(Q2)}~Does the two-stage framework yield stronger guarantees and higher coverage than single-stage baselines under multiple-$m$ setting (\Cref{sec:exp-main})?
\textbf{(Q3)}~Can the framework extend to cascaded multi-model architectures to further improve coverage and guarantee success rate (\Cref{sec:exp-cascade})?

\subsection{Experimental Setup}
\label{sec:exp-setup}

We first describe the datasets, models, and evaluation settings shared across all experiments.

\paragraph{Datasets.}
We evaluate our framework on four preference-evaluation benchmarks that cover diverse domains:
\textbf{(1)~TL;DR}~\citep{stiennon2020learning}, a summarization dataset where the judge selects the best summary from $m$ candidates;
\textbf{(2)~Chatbot Arena}~\citep{zheng2023judging}, comprising real-world user interactions with chatbot systems;
\textbf{(3)~HH-RLHF}~\citep{bai2022training}, covering helpfulness and harmlessness preferences;
\textbf{(4)~AlpacaEval}~\citep{dubois2023alpacafarm}, an instruction-following evaluation benchmark.
For each dataset we construct multi-candidate instances with $m \in \{5, 10, 20\}$ candidates per query, yielding approximately 3{,}000 instances per $(dataset, m)$ pair.
All datasets include human-annotated ground-truth preferences.
The detailed procedure for constructing multi-candidate instances is provided in \Cref{app:data-construction}.

\paragraph{Judge Models.}
We employ a range of open-weight LLMs spanning different parameter scales, enabling both single-model evaluation and multi-model cascading:
\textbf{Mistral-7B-Instruct-v0.2}~(7B),
\textbf{Llama-3-8B-Instruct} (8B),
\textbf{Qwen2.5-7B-Instruct} (7B),
\textbf{DeepSeek-V2-Lite-Chat}~(16B),
\textbf{Qwen2.5-32B-Instruct}~(32B),
\textbf{DeepSeek-67B-Chat}~(67B),
\textbf{Llama-3-70B-Instruct}~(70B),
\textbf{Qwen2.5-72B-Instruct}~(72B), and
\textbf{GPT-OSS-120B}~(120B).
This covers five model families (Mistral, Llama, Qwen, DeepSeek, GPT) with a parameter-size gradient from 7B to 120B, provides sufficient diversity for constructing multiple cascaded configurations.
All judge models are evaluated on the same setting of instances per dataset, ensuring data alignment across models so that any cascade configuration can be composed without missing instances.

\paragraph{Confidence Estimation.}
Following \citet{jung2024trust}, we use Simulated Annotators with $K{=}5$ few-shot examples per $N{=}5$ simulated annotators.
For each instance, we compute and store the full output probability distributions, allowing all downstream scoring functions to be evaluated directly.

\paragraph{Scoring (Confidence) Functions.}
For Stage~I conformal localization, we adopt the Ensemble Mean Probability (EMP) as the scoring function $\mathbb{S}_c$. 
For Stage~II selective auto-picking, we use the KL-divergence as the default scoring function $\mathbb{S}_h$, and compute the margin according to \eqref{eq:marginS}. 
We also evaluate several alternative scoring functions in Stage~II, including non-margin variants that directly use the raw score without computing the margin in \eqref{eq:marginS}. 
The full definitions of all scoring functions are provided in \Cref{app:scoring}.

\paragraph{Evaluation Protocol.}
For each (dataset, model) combination, we repeat the following procedure 1{,}000 times: randomly split the data into calibration and test sets with a (50\%/50\%) ratio, apply our framework (or the single-stage baseline of \citet{jung2024trust}) using calibration set, and evaluate performance on test set.
We report two metrics:
\begin{itemize}[nosep,leftmargin=*]
    \item \textbf{Coverage Rate (Cov.)}: the proportion of test instances that are accepted by the system, averaged over the 1{,}000 random splits. Higher values indicate that the method can safely make decisions on a larger fraction of inputs.
    \item \textbf{Guarantee Success Rate (GSR)}: the proportion of the 1{,}000 splits in which the empirical human–LLM agreement on accepted test instances exceeds the target guarantee level (for our two-stage framework, this is $(1-\varepsilon)(1-\alpha)$). 
    Ideally, this proportion should be at least the confidence level of the guarantee, i.e., $1-\delta$.
\end{itemize}

\begin{table*}[t]
\centering
\scriptsize
\resizebox{\textwidth}{!}{%
\begin{tabular}{cl ccc ccc ccc}
\toprule
\multirow{2.5}{*}{Dataset} & \multirow{2.5}{*}{Judge} & \multicolumn{3}{c}{Direct $m$-to-$1$} & \multicolumn{3}{c}{Stage~I ($m$-to-$k$)} & \multicolumn{3}{c}{Stage~II ($k$-to-$1$)} \\
\cmidrule(lr){3-5}\cmidrule(lr){6-8}\cmidrule(lr){9-11}
& & $m{=}5$ & $m{=}10$ & $m{=}20$ & $m{=}5$ & $m{=}10$ & $m{=}20$ & $m{=}5$ & $m{=}10$ & $m{=}20$ \\
\midrule
\multirow{3}{*}{TL;DR}
& Llama-3-8B    & 0.21 & 0.32 & 0.43 & 0.03 & 0.05 & 0.07 & 0.02 & 0.04 & 0.05 \\
& Qwen2.5-72B   & 0.14 & 0.24 & 0.34 & 0.01 & 0.02 & 0.04 & 0.00 & 0.01 & 0.02 \\
& GPT-OSS-120B  & 0.11 & 0.20 & 0.30 & 0.01 & 0.02 & 0.03 & 0.00 & 0.01 & 0.02 \\
\midrule
\multirow{3}{*}{\shortstack[l]{Chatbot\\Arena}}
& Llama-3-8B    & 0.24 & 0.35 & 0.46 & 0.04 & 0.06 & 0.08 & 0.03 & 0.05 & 0.07 \\
& Qwen2.5-72B   & 0.15 & 0.26 & 0.37 & 0.02 & 0.03 & 0.05 & 0.01 & 0.02 & 0.03 \\
& GPT-OSS-120B  & 0.12 & 0.22 & 0.33 & 0.01 & 0.02 & 0.04 & 0.01 & 0.02 & 0.03 \\
\midrule
\multirow{3}{*}{HH-RLHF}
& Llama-3-8B    & 0.19 & 0.30 & 0.40 & 0.03 & 0.05 & 0.07 & 0.02 & 0.03 & 0.05 \\
& Qwen2.5-72B   & 0.14 & 0.21 & 0.31 & 0.01 & 0.02 & 0.03 & 0.00 & 0.01 & 0.02 \\
& GPT-OSS-120B  & 0.10 & 0.18 & 0.27 & 0.01 & 0.01 & 0.03 & 0.00 & 0.01 & 0.01 \\
\midrule
\multirow{3}{*}{AlpacaEval}
& Llama-3-8B    & 0.22 & 0.34 & 0.45 & 0.04 & 0.06 & 0.08 & 0.03 & 0.05 & 0.06 \\
& Qwen2.5-72B   & 0.14 & 0.23 & 0.35 & 0.01 & 0.02 & 0.04 & 0.00 & 0.02 & 0.03 \\
& GPT-OSS-120B  & 0.11 & 0.19 & 0.31 & 0.01 & 0.02 & 0.03 & 0.00 & 0.01 & 0.02 \\
\bottomrule
\end{tabular}%
}
\vspace{-3mm}
\caption{\textbf{Ranking loss ($\mathcal{L}_{\mathrm{rank}}$, $\downarrow$) measuring confidence--agreement monotonicity on real data} ($T{=}20$ thresholds, $\alpha{=}0.10$).
Direct $m$-to-$1$: single-stage selection using top-1 EMP confidence.
Stage~I ($m$-to-$k$): top-$k$ localization agreement using top-1 EMP confidence.
Stage~II ($k$-to-$1$): selection within the shortlist using KL margin.
}
\label{tab:ranking-loss}
\vspace{-4mm}
\end{table*}

\subsection{Validating Monotonicity across Datasets and Judge LLMs}
\label{sec:exp-motivation}

The simulations in \Cref{fig:motivation} suggest that single-stage monotonicity degrades with $m$ but can be recovered by two-stage decomposition.
We now quantify this phenomenon on real datasets using the ranking loss across multiple models and datasets.

\paragraph{Monotonicity metric.}
Given $T$ evenly spaced confidence thresholds $\lambda_1>\cdots>\lambda_T$ and the corresponding agreement rates $a_1,\dots,a_T$, monotonicity requires $a_1\ge\cdots\ge a_T$.
We measure violations by the ranking loss:
$\mathcal{L}_{\mathrm{rank}} = \sum_{i<j} \mathbf{1}[a_i < a_j]/\binom{T}{2}$,
where $\mathcal{L}_{\mathrm{rank}}{=}0$ indicates perfect monotonicity and $\mathcal{L}_{\mathrm{rank}}{\approx}0.5$ indicates random ordering.
We evaluate three settings corresponding to the panels of \Cref{fig:motivation}: Direct $m$-to-$1$, Stage~I ($m$-to-$k$), and Stage~II ($k$-to-$1$). 
Details of the setting are given in \Cref{app:conf-dist}.

\Cref{tab:ranking-loss} reports ranking loss across four datasets, three judge models, and three candidate-set sizes.
\textbf{(1)}~Single-stage ranking loss increases substantially with $m$ (e.g., $0.14 \to 0.34$ for Qwen2.5-72B on TL;DR), confirming the monotonicity breakdown shown in \Cref{fig:motivation}(a).
\textbf{(2)}~Both Stage~I and Stage~II maintain low ranking loss ($\leq 0.08$) across all settings, validating that the two-stage decomposition restores the monotonic structure required for reliable threshold calibration.
The confidence--agreement curves in \Cref{fig:real-motivation} visually confirm this contrast: the single-stage curves degrade with $m$ while both stages of our decomposition remain monotonic.
Furthermore, we provide confidence-score distributions (see \Cref{app:conf-dist}) for additional supporting evidence.

\begin{figure*}[t]
    \centering
    \includegraphics[width=\textwidth]{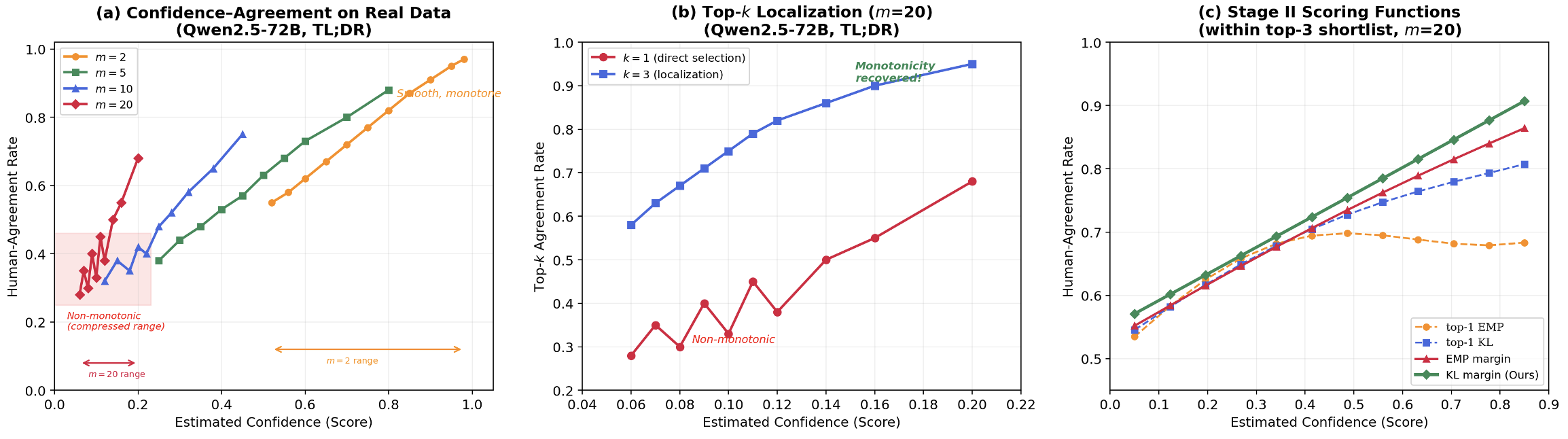}
    \vspace{-8mm}
    \caption{\textbf{Confidence--agreement curves on TL;DR and Qwen2.5-72B.}
    \textbf{(a)}~Single-stage $m$-to-$1$.
    \textbf{(b)}~Stage~I top-$k$ localization.
    \textbf{(c)}~Stage~II scoring within the top-3 shortlist ($m{=}20$).
    }
    \label{fig:real-motivation}
\vspace{-1mm}
\end{figure*}

\begin{table*}[t]
\centering
\footnotesize
\resizebox{\textwidth}{!}{%
\begin{tabular}{cl ccc ccc ccc ccc}
\toprule
\multirow{2.5}{*}{Judge} & \multirow{2.5}{*}{Method} & \multicolumn{3}{c}{TL;DR} & \multicolumn{3}{c}{Chatbot Arena} & \multicolumn{3}{c}{HH-RLHF} & \multicolumn{3}{c}{AlpacaEval} \\
\cmidrule(lr){3-5}\cmidrule(lr){6-8}\cmidrule(lr){9-11}\cmidrule(lr){12-14}
& & Cov. & GSR & Agr. & Cov. & GSR & Agr. & Cov. & GSR & Agr. & Cov. & GSR & Agr. \\
\midrule
\multirow{8}{*}{\shortstack[l]{Llama-3\\8B}}
& No Selection                   & 100\%& 0\% & 36\%\std{1.2} & 100\%& 0\% & 31\%\std{1.4} & 100\%& 0\% & 40\%\std{1.1} & 100\%& 0\% & 33\%\std{1.3} \\
& Single-Stage (top-1 EMP)                   & 26\%\std{3.2} & 53.8\% & 77\%\std{2.3} & 21\%\std{3.6} & 47.6\% & 74\%\std{2.5} & 30\%\std{2.9} & 58.4\% & 79\%\std{2.0} & 25\%\std{3.4} & 52.1\% & 76\%\std{2.2} \\
& Single-Stage (EMP margin)            & 31\%\std{3.0} & 58.6\% & 78\%\std{2.1} & 26\%\std{3.4} & 53.2\% & 76\%\std{2.3} & 35\%\std{2.7} & 63.5\% & 80\%\std{1.9} & 28\%\std{3.2} & 55.8\% & 78\%\std{2.2} \\
& Single-Stage (KL margin)              & 35\%\std{2.8} & 63.1\% & 80\%\std{2.0} & 30\%\std{3.1} & 57.8\% & 78\%\std{2.1} & 39\%\std{2.5} & 68.3\% & 82\%\std{1.8} & 33\%\std{3.0} & 61.7\% & 79\%\std{2.1} \\
& Single-Stage (Vote)                   & 33\%\std{2.9} & 60.7\% & 79\%\std{2.0} & 28\%\std{3.2} & 55.4\% & 77\%\std{2.2} & 37\%\std{2.7} & 65.2\% & 81\%\std{1.8} & 31\%\std{3.0} & 59.3\% & 79\%\std{2.0} \\
\cmidrule(l){2-14}
\rowcolor{blue!6} \cellcolor{white} & \textbf{Two-Stage (EMP margin)}     & 39\%\std{2.6} & \textbf{91.2\%} & 84\%\std{1.6} & 34\%\std{2.9} & \textbf{90.4\%} & 82\%\std{1.8} & 43\%\std{2.4} & \textbf{92.6\%} & 85\%\std{1.4} & 37\%\std{2.8} & \textbf{90.8\%} & 84\%\std{1.7} \\
\rowcolor{blue!6} \cellcolor{white} & \textbf{Two-Stage (KL margin)}       & 45\%\std{2.3} & \textbf{95.1\%} & 86\%\std{1.4} & 40\%\std{2.6} & \textbf{93.7\%} & 85\%\std{1.6} & 49\%\std{2.2} & \textbf{96.4\%} & 87\%\std{1.2} & 43\%\std{2.5} & \textbf{94.3\%} & 86\%\std{1.5} \\
\rowcolor{blue!6} \cellcolor{white} & \textbf{Two-Stage (Vote)}            & 42\%\std{2.5} & \textbf{93.4\%} & 85\%\std{1.5} & 37\%\std{2.7} & \textbf{91.8\%} & 84\%\std{1.7} & 46\%\std{2.3} & \textbf{94.1\%} & 86\%\std{1.3} & 40\%\std{2.6} & \textbf{92.5\%} & 85\%\std{1.6} \\
\midrule
\multirow{8}{*}{\shortstack[l]{Qwen2.5\\72B}}
& No Selection                   & 100\%& 0\% & 52\%\std{1.1} & 100\%& 0\% & 48\%\std{1.3} & 100\%& 0\% & 54\%\std{1.0} & 100\%& 0\% & 51\%\std{1.2} \\
& Single-Stage (top-1 EMP)                   & 41\%\std{2.7} & 62.4\% & 80\%\std{1.9} & 35\%\std{3.0} & 56.8\% & 78\%\std{2.1} & 43\%\std{2.5} & 65.7\% & 82\%\std{1.7} & 40\%\std{2.8} & 61.3\% & 80\%\std{2.0} \\
& Single-Stage (EMP margin)            & 46\%\std{2.5} & 67.3\% & 82\%\std{1.8} & 40\%\std{2.8} & 62.1\% & 80\%\std{2.0} & 48\%\std{2.4} & 70.8\% & 83\%\std{1.6} & 45\%\std{2.6} & 66.5\% & 81\%\std{1.9} \\
& Single-Stage (KL margin)              & 50\%\std{2.4} & 72.6\% & 83\%\std{1.7} & 44\%\std{2.6} & 67.4\% & 81\%\std{1.9} & 52\%\std{2.3} & 75.3\% & 84\%\std{1.5} & 49\%\std{2.5} & 71.2\% & 83\%\std{1.7} \\
& Single-Stage (Vote)                   & 48\%\std{2.5} & 69.8\% & 82\%\std{1.8} & 42\%\std{2.7} & 64.6\% & 80\%\std{1.9} & 50\%\std{2.3} & 73.1\% & 84\%\std{1.6} & 47\%\std{2.5} & 68.7\% & 82\%\std{1.8} \\
\cmidrule(l){2-14}
\rowcolor{blue!6} \cellcolor{white} & \textbf{Two-Stage (EMP margin)}     & 57\%\std{2.1} & \textbf{92.3\%} & 86\%\std{1.3} & 51\%\std{2.4} & \textbf{90.7\%} & 85\%\std{1.5} & 59\%\std{2.0} & \textbf{93.1\%} & 87\%\std{1.1} & 56\%\std{2.2} & \textbf{91.6\%} & 86\%\std{1.4} \\
\rowcolor{blue!6} \cellcolor{white} & \textbf{Two-Stage (KL margin)}       & 64\%\std{1.9} & \textbf{95.8\%} & 88\%\std{1.1} & 58\%\std{2.2} & \textbf{94.6\%} & 87\%\std{1.3} & 66\%\std{1.8} & \textbf{96.2\%} & 89\%\std{1.0} & 63\%\std{2.0} & \textbf{95.1\%} & 88\%\std{1.2} \\
\rowcolor{blue!6} \cellcolor{white} & \textbf{Two-Stage (Vote)}            & 61\%\std{2.0} & \textbf{94.2\%} & 87\%\std{1.2} & 55\%\std{2.3} & \textbf{92.9\%} & 86\%\std{1.4} & 63\%\std{1.9} & \textbf{94.7\%} & 88\%\std{1.0} & 60\%\std{2.1} & \textbf{93.6\%} & 87\%\std{1.3} \\
\midrule
\multirow{8}{*}{\shortstack[l]{GPT-OSS\\120B}}
& No Selection                   & 100\%& 0\% & 58\%\std{1.0} & 100\%& 0\% & 54\%\std{1.2} & 100\%& 0\% & 60\%\std{0.9} & 100\%& 0\% & 57\%\std{1.1} \\
& Single-Stage (top-1 EMP)                   & 47\%\std{2.5} & 66.3\% & 82\%\std{1.8} & 41\%\std{2.8} & 61.4\% & 80\%\std{2.0} & 49\%\std{2.3} & 69.7\% & 83\%\std{1.6} & 46\%\std{2.6} & 65.2\% & 81\%\std{1.9} \\
& Single-Stage (EMP margin)            & 52\%\std{2.3} & 71.8\% & 83\%\std{1.7} & 46\%\std{2.6} & 66.3\% & 81\%\std{1.9} & 54\%\std{2.1} & 74.6\% & 85\%\std{1.4} & 51\%\std{2.4} & 70.4\% & 83\%\std{1.7} \\
& Single-Stage (KL margin)              & 56\%\std{2.1} & 76.4\% & 85\%\std{1.5} & 50\%\std{2.4} & 71.7\% & 83\%\std{1.7} & 58\%\std{2.0} & 79.3\% & 86\%\std{1.3} & 55\%\std{2.2} & 74.8\% & 84\%\std{1.6} \\
& Single-Stage (Vote)                   & 54\%\std{2.2} & 73.6\% & 84\%\std{1.6} & 48\%\std{2.5} & 68.9\% & 82\%\std{1.8} & 56\%\std{2.0} & 76.8\% & 85\%\std{1.4} & 53\%\std{2.3} & 72.3\% & 84\%\std{1.6} \\
\cmidrule(l){2-14}
\rowcolor{blue!6} \cellcolor{white} & \textbf{Two-Stage (EMP margin)}     & 63\%\std{1.9} & \textbf{93.4\%} & 87\%\std{1.2} & 57\%\std{2.2} & \textbf{91.8\%} & 86\%\std{1.4} & 65\%\std{1.8} & \textbf{94.3\%} & 88\%\std{1.0} & 62\%\std{2.0} & \textbf{92.7\%} & 87\%\std{1.3} \\
\rowcolor{blue!6} \cellcolor{white} & \textbf{Two-Stage (KL margin)}       & 70\%\std{1.7} & \textbf{96.8\%} & 89\%\std{1.0} & 64\%\std{2.0} & \textbf{95.4\%} & 88\%\std{1.2} & 72\%\std{1.6} & \textbf{97.1\%} & 90\%\std{0.9} & 69\%\std{1.8} & \textbf{96.3\%} & 89\%\std{1.1} \\
\rowcolor{blue!6} \cellcolor{white} & \textbf{Two-Stage (Vote)}            & 67\%\std{1.8} & \textbf{95.2\%} & 88\%\std{1.1} & 61\%\std{2.1} & \textbf{93.6\%} & 87\%\std{1.3} & 69\%\std{1.7} & \textbf{95.7\%} & 89\%\std{0.9} & 66\%\std{1.9} & \textbf{94.8\%} & 88\%\std{1.2} \\
\bottomrule
\end{tabular}%
}
\vspace{-3mm}
\caption{\textbf{Two-stage vs.\ single-stage across judge models and datasets} ($m{=}10$, target agreement $= 0.81$, $\delta{=}0.10$).
Two-stage: $\varepsilon{=}\alpha{=}0.10$; single-stage: $\varepsilon'{=}0.19$.
For Two-Stage framework, Stage~I always uses EMP scores.
Agr.\ denotes average empirical human-agreement on accepted instances.
Cov. and Agr.\ are reported as mean $\pm$ std over 1{,}000 random splits.
We provide more results in \Cref{app:full-results}.
}
\label{tab:main}
\vspace{-3mm}
\end{table*}

\begin{figure*}[t]
    \centering
    \includegraphics[width=\textwidth]{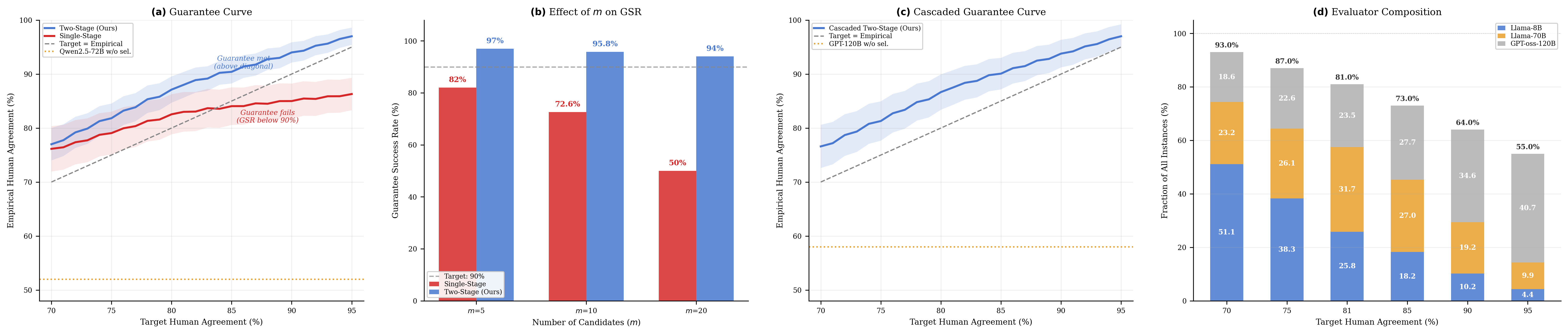}
    \vspace{-7mm}
    \caption{\textbf{Guarantee and cascade results} (TL;DR dataset).
    \textbf{(a)}~Single-model guarantee curve (Qwen2.5-72B, $m{=}10$): Two-Stage (blue) consistently meets the target (above diagonal); Single-Stage (red) falls below. Shaded bands show min/max over 1{,}000 splits.
    \textbf{(b)}~Effect of $m$ on GSR (Qwen2.5-72B): Single-Stage drops from 82\% to 50\% as $m$ grows; Two-Stage remains $\geq 91\%$.
    \textbf{(c)}~Cascaded guarantee curve (Llama-8B $\to$ 70B $\to$ GPT-OSS-120B, $m{=}10$): empirical agreement tracks the target throughout.
    \textbf{(d)}~Evaluator composition (same cascade as (c)) as fraction of all test instances: at lower targets the weakest tier handles most instances; as the target tightens, more instances escalate to stronger tiers.
    }
    \label{fig:combined-results}
\vspace{-3mm}
\end{figure*}

\subsection{Two-Stage vs.\ Single-Stage}
\label{sec:exp-main}

Having confirmed that the two-stage decomposition restores monotonicity, we now test whether this translates into practical guarantee improvements.
We use three scoring (confidence) functions: EMP margin, KL margin, and Vote agreement (defined in \Cref{app:scoring}), with both a \textbf{Single-Stage} and a \textbf{Two-Stage} framework, along with No Selection and Single-Stage~\citep{jung2024trust} as references.

\noindent
\textbf{Two-stage achieves valid guarantees with higher coverage.}
\Cref{tab:main} reports coverage and GSR at $m{=}10$ (target agreement $= 0.81$, $\delta{=}0.10$) across four datasets and three representative judges from different model families (8B--120B), averaged over 1{,}000 random splits.
Without any selection, even the strongest judge (GPT-OSS-120B) achieves only 54--60\% agreement, well below the 81\% target, confirming that selective evaluation is necessary.
Across all scoring functions and datasets, every two-stage variant exceeds the $1{-}\delta = 90\%$ GSR target (e.g., 95--97\% for KL margin on GPT-OSS-120B), while its single-stage counterpart using the same scoring function remains far below (48--79\%).
The improvement extends to coverage as well: Two-Stage~(KL margin) with GPT-OSS-120B covers 64--72\% of instances versus 50--58\% for Single-Stage~(KL margin), and even for the smallest judge (Llama-3-8B) two-stage raises coverage by 8--10\%.
Since matched Single-/Two-Stage pairs differ only in Stage~I, the gain is architecture-driven rather than scoring-specific.
Targeted comparisons with a variance-normalized adaptation of \citet{jung2024trust} and component-wise KL-margin ablations are provided in \Cref{app:additional-validation}.

\noindent
\textbf{The performance of two-stage framework is robust across target levels, candidate-set sizes, and random splits.}
\Cref{fig:combined-results}(a) sweeps the target agreement continuously from 70\% to 95\% at $m{=}10$, showing whether the guarantee holds across the full target range rather than at one point.
The two-stage empirical agreement (blue) tracks the target throughout, whereas the single-stage curve (red) falls below for most targets (guarantee fails).
The shaded bands (min/max over 1{,}000 splits) further confirm that two-stage calibration is substantially more stable (e.g., $\pm 2.1\%$ vs.\ $\pm 4.7\%$ spread at target 85\% for Qwen2.5-72B).
\Cref{fig:combined-results}(b) directly verifies our core motivation by varying $m$ from 5 to 20: single-stage GSR degrades sharply from 82\% to 50\%, showing that the single-stage failure is caused by increasing $m$; in contrast, two-stage GSR remains above the $1{-}\delta{=}90\%$ threshold throughout, validating the robustness of our two-stage framework against the $m$ increasing.
More results for all judges are provided in \Cref{app:full-results}.

\subsection{Cascaded Architecture}
\label{sec:exp-cascade}

The results above establish that our two-stage framework provides valid guarantees for any individual judge.
In practice, however, deploying a single large model for all instances is expensive.
A natural strategy is to cascade judges of increasing capacity: a weaker model evaluates first; only when it abstains does the instance escalate to a stronger one.
\citet{jung2024trust} demonstrate this for $m{=}2$, but single-stage guarantees break down for $m>2$ (\Cref{sec:exp-main}).
Because our two-stage framework restores valid per-tier guarantees for $m>2$, it naturally enables cascaded architectures with provable reliability.
We evaluate five cascade configurations (\Cref{tab:cascade}), including same-family, cross-family, and two-tier setups, along with two single-model references, all under the same protocol as \Cref{sec:exp-main}.
Each tier operates as a complete two-stage selective evaluator; an instance is accepted only when the tier's Stage~II confidence exceeds its calibrated threshold, and escalated otherwise.

\begin{table*}[t]
\centering
\setlength{\tabcolsep}{13.5pt}
\scriptsize
\begin{tabular}{l cc ccccc}
\toprule
\multirow{4}{*}{Cascade Configuration} & \multicolumn{2}{c}{Single-Stage} & \multicolumn{5}{c}{\textbf{Two-Stage (Ours)}} \\
\cmidrule(lr){2-3}\cmidrule(lr){4-8}
& \multirow{2.5}{*}{Cov.} & \multirow{2.5}{*}{GSR} & \multirow{2.5}{*}{Cov.} & \multirow{2.5}{*}{GSR} & \multicolumn{3}{c}{Comp.\ (\%)} \\
\cmidrule(lr){6-8}
& & & & & $T_1$ & $T_2$ & $T_3$ \\
\midrule
Qwen2.5-72B only                             & 50\%\std{2.4} & 72.6\%  & \cellcolor{blue!6} 64\%\std{1.9} &\cellcolor{blue!6} 95.8\%  & --  & --  & 100.0 \\
GPT-oss-120B only                            & 56\%\std{2.1} & 76.4\%  &\cellcolor{blue!6} 70\%\std{1.7} &\cellcolor{blue!6} 96.8\%  & --  & --  & 100.0 \\
\midrule
Qwen-7B $\to$ 32B $\to$ 72B                 & 60\%\std{2.8} & 65.7\%  &\cellcolor{blue!6} 78\%\std{1.9} &\cellcolor{blue!6} \textbf{93.1\%}  & 37.8 & 35.6 & 26.5 \\
Llama-8B $\to$ 70B $\to$ GPT-120B           & 63\%\std{2.5} & 68.3\%  &\cellcolor{blue!6} 81\%\std{1.8} &\cellcolor{blue!6} \textbf{94.2\%}  & 31.9 & 39.1 & 29.0 \\
Mistral-7B $\to$ DS-67B $\to$ GPT-120B      & 58\%\std{3.1} & 64.8\%  &\cellcolor{blue!6} 77\%\std{2.2} &\cellcolor{blue!6} \textbf{91.6\%}  & 35.3 & 38.2 & 26.5 \\
Llama-8B $\to$ Qwen-32B $\to$ 72B           & 61\%\std{2.7} & 66.9\%  &\cellcolor{blue!6} 79\%\std{2.0} &\cellcolor{blue!6} \textbf{92.8\%}  & 34.9 & 38.1 & 27.0 \\
DeepSeek-16B $\to$ Qwen-72B                 & 57\%\std{3.2} & 63.5\%  &\cellcolor{blue!6} 74\%\std{2.3} &\cellcolor{blue!6} \textbf{91.4\%}  & 45.7 & 54.3 & -- \\
\bottomrule
\end{tabular}
\vspace{-3mm}
\caption{\textbf{Cascaded evaluation on TL;DR} ($m{=}10$, target agreement $= 0.81$, $\delta{=}0.10$).
Instances abstained by one tier escalate to the next.
$T_i$ denotes the $i$-th model in the cascade (weakest first);
``Comp.'' shows each tier's share of the covered instances (summing to 100\%).
Coverage and GSR are reported as mean $\pm$ std over 1{,}000 random splits.
}
\label{tab:cascade}
\vspace{-3mm}
\end{table*}

\noindent
\textbf{Two-stage cascades maintain guarantees; single-stage cascades do not.}
\Cref{tab:cascade} compares both variants side by side.
Without two-stage decomposition, every cascade fails to reach the $1{-}\delta = 90\%$ GSR threshold (64--76\%), consistent with the per-model failures in \Cref{tab:main}; simply adding more models does not fix the monotonicity breakdown at each tier.
In contrast, every two-stage cascade exceeds 90\% GSR (91--95\%), confirming that valid per-tier guarantees compose correctly.
Cascading also substantially improves coverage: the Qwen same-family cascade reaches 78\% versus 64\% for Qwen2.5-72B alone (+14\%), and across all three-tier configurations two-stage cascades cover 77--81\% of instances compared with 64--70\% for the strongest single model.
The composition columns show that the strongest model is invoked for fewer than 30\% of instances in every three-tier cascade, demonstrating significant computational savings.

\paragraph{Guarantee curve and evaluator composition.}

\Cref{fig:combined-results}(c)--(d) illustrates a representative cascade (Llama-8B $\to$ 70B $\to$ GPT-120B).
In (c), empirical agreement tracks the target over the 70--95\% range, showing that the cascaded guarantee remains stable across target levels.
Panel (d) shows that the cascade covers 93.0\% of instances at the 70\% target, with the weakest tier handling over half; as the target tightens, total coverage decreases and more instances are routed to stronger tiers.
Additional cascade results and related work are provided in \Cref{app:full-results,app:related}.

\section{Conclusion}

In this work, we studied the problem of providing reliable human–LLM agreement guarantees for LLM-based evaluators in realistic multi-candidate judgment settings, where a judge must select a single preferred response from several alternatives. 
While prior confidence-thresholding approaches provide guarantees in the pairwise case, their key monotonicity assumption between confidence and disagreement risk breaks down as the number of candidates increases, making existing guarantees unreliable in practical evaluation pipelines.
To address this limitation, we introduced a Localize-Then-Decide framework, which decomposes the multi-candidate decision into two calibrated stages. 
A conformal prediction step first localizes a small shortlist that contains the human-preferred response with high probability, and a selective decision rule then chooses a single response from this shortlist or abstains when reliability cannot be ensured. 
This decomposition restores the confidence–agreement monotonicity required for valid threshold calibration and yields finite-sample guarantees on selective human–LLM agreement.
Empirical results across multiple datasets, judge models, and candidate sizes demonstrate that the proposed framework consistently achieves valid guarantees while improving coverage compared with single-stage baselines. 
In addition, the framework naturally enables cascaded evaluation architectures, allowing weaker models to handle easy cases while escalating difficult instances to stronger models without violating reliability guarantees.

Overall, our theoretical and empirical results suggest that structuring LLM judgment as a localized and selectively certified decision process is key to achieving trustworthy automated evaluation. 
By extending provable guarantees from pairwise comparisons to realistic multi-generation scenarios, this work takes an important step toward reliable and scalable LLM-as-a-judge systems.

\clearpage

\section*{Limitations}

As with all conformal-prediction-based methods, our theoretical guarantees require the calibration and test instances to be exchangeable (Assumption~\ref{ass:exch-c}).
When the evaluation domain at deployment time diverges from the calibration corpus (e.g., due to topic shift or prompt reformulation), the finite-sample bounds may become less tight~\citep{gibbs2021adaptive}.
While fully relaxing this assumption remains an open problem shared by the broader conformal prediction literature, our consistent results across four benchmarks spanning diverse domains (summarization, open-ended dialogue, safety, and instruction-following) suggest that the framework is applicable to a wide variety of evaluation scenarios when calibration data is drawn from the target domain.

Our confidence estimation builds on Simulated Annotators~\citep{jung2024trust}, requiring $N$ forward passes per instance (5 in our default configuration), each using a prompt with $K{=}5$ few-shot examples; this is more expensive than obtaining a single predictive probability.
However, the two-stage framework itself is agnostic to the choice of confidence estimator; Simulated Annotators can be replaced by lighter-weight alternatives (e.g., a single-pass predictive probability or verbalized confidence) whenever the resulting scores preserve sufficient discriminative power for threshold calibration.

Finally, our evaluation focuses on the preference-judgment setting where the judge selects a single preferred response from multiple candidates.
While this covers a wide range of practical LLM evaluation scenarios (e.g., best-of-$N$ decoding, multi-system comparison), other evaluation paradigms such as Likert-scale scoring or factuality verification involve different output structures and may require task-specific adaptations of the scoring and calibration procedures.
Nonetheless, the core Localize-Then-Decide principle is general and could serve as a foundation for extending formal guarantees to these broader settings.

\section*{Ethics Statement}

This work aims to improve the reliability of LLM-based evaluation by providing formal statistical guarantees on human--LLM agreement.
All experiments use publicly available datasets with existing human-preference annotations; no new human data was collected for this study.
We note that the guarantees provided by our framework are statistical in nature and hold under the stated assumptions. They should not be interpreted as absolute correctness certificates, and human oversight remains important in high-stakes evaluation scenarios.

\section*{Acknowledgments}

This work was supported by the NVIDIA Academic Grant Program (Exploiting Overthinking Attacks on GenAI), the Royal Society Grant (Ensuring Trustworthy AI: Robustness Certification for Large Language Models) [Reference RGS\textbackslash{}R2\textbackslash{}252444], and the AIRR Gateway project (Exploiting Robustness of Reasoning Efficiency in Agentic AI).

\bibliography{custom}

\clearpage
\appendix

\section{Related Work}
\label{app:related}

\paragraph{LLM-as-a-judge.}
LLM-based evaluation has emerged as a practical alternative to costly human annotation, powering automated benchmarks~\citep{zheng2023judging,dubois2023alpacafarm} and preference data collection at scale~\citep{chiang2023can}.
Subsequent work enriches the evaluation paradigm along multiple axes:
chain-of-thought scoring~\citep{liu2023g}, sub-criteria decomposition~\citep{saha2024branch}, and fine-grained skill-based assessment~\citep{ye2023flask} improve prompting strategies;
dedicated judge models such as Prometheus~\citep{kim2023prometheus,kim2024prometheus}, Auto-J~\citep{li2023generative}, JudgeLM~\citep{zhu2023judgelm}, and FLAMe~\citep{vu2024foundational} improve alignment with human judgments through task-specific fine-tuning;
and multi-judge approaches such as PoLL~\citep{verga2024replacing} and Product-of-Experts~\citep{liusie2024efficient} reduce single-model cost and bias.
Despite these advances, LLM judges exhibit well-documented failure modes, including position bias~\citep{zheng2023judging,wang2024large,shi2025judging,li2023split}, self-preference bias~\citep{panickssery2024llm,stureborg2024large}, length and verbosity biases~\citep{wu2025style,li2025generation,chen2024humans}, with benchmark audits~\citep{tan2024judgebench,lambert2025rewardbench,shankar2024validates} further exposing reliability gaps.
While the above methods improve accuracy or reduce biases empirically, none provides a formal guarantee on the resulting error rate.
Furthermore, the vast majority of these studies operate in the pairwise regime ($m{=}2$), whereas practical pipelines such as best-of-$N$ decoding in RLHF~\citep{stiennon2020learning,ouyang2022training,rafailov2023direct,cobbe2021training,lightman2023let} and multi-system comparisons require judging among $m{>}2$ candidates, a setting where probability dilution introduces additional reliability challenges.

\paragraph{Selective prediction and conformal risk control.}
This absence of formal guarantees motivates tools from selective prediction and conformal inference.
Selective prediction~\citep{el2010foundations,geifman2017selective} trades coverage for reliability by allowing abstention under high uncertainty.
Conformal prediction~\citep{vovk2005algorithmic,shafer2008tutorial,angelopoulos2023conformal,romano2019conformalized,lei2021conformal} and distribution-free risk control~\citep{bates2021distribution,angelopoulos2022conformal,angelopoulos2025learn,gibbs2021adaptive} provide principled frameworks for calibrating such abstention thresholds with finite-sample guarantees.
These tools have been adapted to LLM settings for calibrated abstention~\citep{gui2024conformal}, token-level prediction sets~\citep{quach2023conformal}, multi-choice QA~\citep{kumar2023conformal}, hallucination mitigation~\citep{yadkori2024mitigating,manakul2023selfcheckgpt}, and factuality guarantees~\citep{mohri2024language}, while complementary confidence-calibration techniques~\citep{naeini2015obtaining,guo2017calibration,kadavath2022language,wang2022self,xiong2023can,tian2023just,detommaso2024multicalibration,geng2024survey,lin2024generating,ye2024benchmarking,khanmohammadi2025calibrating} improve the underlying uncertainty estimates.
While these methods provide formal guarantees for LLM generation and prediction tasks, they do not directly address the LLM-as-a-judge evaluation setting where the label space grows with $m$.

\paragraph{Guarantees for LLM evaluation.}
A separate line of work brings formal reliability guarantees into the LLM evaluation setting specifically.
Most notably, \citet{jung2024trust} propose Cascaded Selective Evaluation, which provides provable human-agreement guarantees for pairwise LLM judging ($m{=}2$) via fixed-sequence testing and introduces Simulated Annotators for confidence estimation.
Concurrently, \citet{badshah2026scope} propose SCOPE, which improves pairwise reliability through Bidirectional Preference Entropy (BPE), a bias-neutral uncertainty estimator that mitigates position bias.
In a complementary direction, \citet{park2025adaptive} study prediction-powered inference for automated evaluation with reliability guarantees.
These works collectively demonstrate the growing need for principled reliability frameworks in LLM evaluation, yet all existing guarantee methods are restricted to the pairwise setting ($m{=}2$)~\citep{jin2026marginadaptive}.

\paragraph{Broader reliability landscape.}
Beyond evaluation-specific guarantees, research on trustworthy GenAI examines adversarially induced overthinking and reasoning-level denial-of-service attacks~\citep{li2025pot,liu2026badthink,li2026otora}, alongside systematic guardrail design~\citep{pmlr-v235-dong24c} and certified LLM robustness~\citep{wang2026clucert}.
Complementary statistical work develops PAC-Bayesian generalization analyses based on weight correlation~\citep{jin2020does}, second-order weight statistics~\citep{jin2022enhancing}, and the spectral norm of robust confusion matrices~\citep{jin2025enhancing}.

\paragraph{Distinction from prior work.}
Our work extends the guarantee-based line of research from the pairwise to the multi-candidate regime ($m{>}2$).
Unlike bias-mitigation methods (e.g., position swapping, multi-judge ensembling, fine-tuned judges) that reduce errors empirically without formal bounds, and unlike existing guarantee methods~\citep{jung2024trust,badshah2026scope} that are restricted to $m{=}2$, we identify a new structural challenge: the monotonic confidence--agreement relationship assumed by prior guarantee methods breaks down when probability mass is diluted across many alternatives.
To address this, we propose a two-stage decomposition, conformal localization followed by selective auto-pick, that restores the monotonic structure and thereby enables valid guarantees for $m{>}2$.
The framework is modular and orthogonal to the choice of scoring function: confidence estimators such as BPE~\citep{badshah2026scope}, Simulated Annotators~\citep{jung2024trust}, or any of the LLM-as-a-judge scoring methods reviewed above can serve as plug-in components within either stage.

\section{Simulation Details}
\label{App:simulation}

The simulation in \Cref{fig:motivation} is designed to illustrate how the confidence--agreement monotonicity degrades as the number of candidates $m$ grows, and how the proposed two-stage decomposition recovers it.
We describe the generative model and the procedure used to produce each panel below.

\paragraph{Generative model.}
We construct a Plackett-Luce-inspired~\citep{plackett1975analysis} generative model that incorporates three realistic sources of variability: (i)~heterogeneous instance difficulty, (ii)~variable candidate quality gaps, and (iii)~diverse annotator signal strengths.
For each simulated instance with $m$ candidates and $N$ annotators, the data generation proceeds as follows.

\begin{enumerate}[nosep,leftmargin=*]
\item \textbf{Instance difficulty.}
Sample a difficulty parameter $d \sim \mathrm{Beta}(2, 5)$, which concentrates most instances on the ``easy'' end ($\mathbb{E}[d]=0.29$) while occasionally producing difficult instances.
The Gumbel noise scale is set to $\sigma = 0.5 + 1.5\,d$, so harder instances produce noisier annotator outputs.

\item \textbf{Candidate quality.}
Sample the quality advantage of the ground-truth best candidate: $q_{\mathrm{gap}} \sim \mathrm{Gamma}(2,\, 0.8)$.
Assign quality $q_1 = q_{\mathrm{gap}}$ to the true best response.
For each remaining candidate $i \in \{2,\dots,m\}$, independently with probability $\rho{=}0.20$, the candidate is competitive: $q_i = q_{\mathrm{gap}} \cdot U(0.6, 1.0) - 0.2$; otherwise it is clearly inferior: $q_i \sim -\mathrm{Gamma}(1.5, 0.5)$.
The competitive-candidate mechanism is a key driver of the monotonicity breakdown: at $m{=}2$ there are on average $0.2$ competitive distractors per instance, whereas at $m{=}20$ there are ${\approx}3.8$, substantially increasing the frequency of high-confidence but incorrect predictions.

\item \textbf{Annotator probabilities.}
For each annotator $j \in \{1,\dots,N\}$, sample a signal strength $s_j \sim \mathrm{Beta}(3, 2)$ (mean $0.6$; strong annotators approach $0.9$, weak ones fall to ${\approx}0.2$).
The perceived quality of candidate $i$ by annotator $j$ is
\[
\tilde{q}_{j,i} \;=\; q_i \cdot s_j \;+\; \epsilon_{j,i},
\qquad
\epsilon_{j,i} \sim \mathrm{Gumbel}(0,\, \sigma),
\]
and the annotator's probability vector is $p_{j,i} = \mathrm{softmax}(\tilde{q}_{j,1},\dots,\tilde{q}_{j,m})_i$.
The multiplicative signal attenuation $q_i \cdot s_j$ compresses the quality gap for weak annotators, making their outputs closer to uniform and thus less informative---analogous to poorly chosen few-shot prompts in real Simulated Annotators.

\item \textbf{Position debiasing.}
Candidate indices are randomly permuted to remove any positional artifact.
The ensemble mean probability (EMP) is then computed as $\bar{p}_i = \frac{1}{N}\sum_{j=1}^{N} p_{j,i}$.
\end{enumerate}

\paragraph{Why this model captures the real phenomenon.}
The interaction between the number of candidates $m$ and the competitive-candidate mechanism creates a realistic probability dilution effect.
As $m$ increases, more competitive distractors appear (on average $\rho(m{-}1)$), splitting the EMP probability mass among confusable alternatives and shrinking the numerical margin between the best and second-best candidates.
This produces the non-monotonic confidence--agreement patterns observed in \Cref{fig:motivation}(a) and motivates the two-stage solution shown in \Cref{fig:motivation}(b)--(c).

\paragraph{Panel descriptions.}
In all three panels of \Cref{fig:motivation}, each point at confidence threshold $\lambda$ reports the average agreement rate over all instances whose estimated confidence satisfies $c(x) \ge \lambda$.
Monotonicity means that this agreement rate is non-decreasing in $\lambda$: a stricter threshold should yield higher accuracy among accepted instances.

\begin{itemize}[nosep,leftmargin=*]
\item \textbf{Panel~(a)} (Confidence--Agreement Monotonicity):
For each $m \in \{2, 5, 10, 20\}$, the confidence score is $c(x) = \max_i \bar{p}_i$ and the agreement indicator is $\mathbf{1}[\arg\max_i \bar{p}_i = y]$.
As $m$ increases, the confidence range compresses (the maximum EMP score decreases due to probability dilution among more candidates) and the monotonic relationship between confidence and agreement progressively degrades.

\item \textbf{Panel~(b)} (Top-$k$ Localization):
Fixing $m{=}20$, we compare direct top-$1$ selection ($k{=}1$) against top-$k$ localization ($k{=}3$).
The confidence score remains $c(x) = \max_i \bar{p}_i$; the agreement indicator changes to $\mathbf{1}[y \in \mathrm{top}\text{-}k(\bar{p})]$, i.e., whether the ground-truth best lies within the top-$k$ candidates ranked by EMP.
While top-$1$ agreement is non-monotonic, top-$3$ agreement recovers a clean monotonic curve with the same confidence score, motivating Stage~I conformal localization.

\item \textbf{Panel~(c)} (Monotonicity Within the Shortlist):
Fixing $m{=}20$ and restricting to instances where $y$ lies in the top-$3$ shortlist $\mathcal{S}$, we compare four confidence functions defined in \Cref{app:scoring}: EMP top-1, KL top-1, EMP margin, and KL margin.
The agreement indicator is $\mathbf{1}[\hat{y} = y \mid y \in \mathcal{S}]$ where $\hat{y}$ is the predicted best within $\mathcal{S}$.
Margin-based confidence functions (especially KL margin) produce substantially smoother and more monotonic reliability curves than top-1 confidence, motivating the margin-based Stage~II design adopted in our framework.
\end{itemize}

\paragraph{Configuration.}
Simulations use $N{=}25$ annotators and $20{,}000$ instances per setting, with agreement rates evaluated at evenly spaced confidence thresholds.
\Cref{tab:sim-config} summarizes the complete parameter settings.

\begin{table}[h]
\centering
\footnotesize
\resizebox{\columnwidth}{!}{%
\begin{tabular}{@{}lll@{}}
\toprule
Parameter & Value & Description \\
\midrule
$m$ & $\{2, 5, 10, 20\}$ & Candidates per instance \\
$N$ & 25 & Simulated annotators \\
Instances & $20{,}000$ & Per $m$ setting \\
$q_{\mathrm{gap}}$ & $\mathrm{Gamma}(2, 0.8)$ & Best-candidate advantage \\
Comp.\ prob.\ $\rho$ & 0.20 & Competitive distractor \\
Non-best $q_i$ & $-\mathrm{Gamma}(1.5, 0.5)$ & Weak-candidate quality \\
Difficulty $d$ & $\mathrm{Beta}(2, 5)$ & Gumbel noise scale \\
Signal $s_j$ & $\mathrm{Beta}(3, 2)$ & Annotator strength \\
Noise $\epsilon$ & $\mathrm{Gumbel}(0, 0.5{+}1.5d)$ & Perception noise \\
\bottomrule
\end{tabular}%
}
\caption{Simulation configuration for \Cref{fig:motivation}.}
\label{tab:sim-config}
\end{table}

\section{Proofs}
\label{App:proof}

\begin{proof}[Proof of Theorem 3.1]

Let $z_1,\dots,z_{n},z_{n+1}$ be exchangeable and define $R_i=R(z_i)$ for $i=1,\dots,n+1$.
By exchangeability, the joint law of $(R_1,\dots,R_{n},R_{n+1})$ is invariant under permutations.
Therefore, the (tie-broken) rank of $R_{n+1}$ among $\{R_1,\dots,R_{n},R_{n+1}\}$ is uniform on
$\{1,\dots,n+1\}$.
Let $k$ be defined by \eqref{eq:k-def-c}. The standard conformal quantile argument yields
$\mathbb{P}(R_{n+1}\le k)\ge 1-\alpha$.
That means $R(z)\le k$ holds if and only if the human-best index lies in the top-$k$ candidates under $\mathbb{S}_c(x,\cdot)$,
i.e., $y\in \mathcal{S}(x)$. Hence \eqref{eq:area-guarantee-c} holds.
The statement $|\mathcal{S}(x)|=k$ follows by construction.
\hfill $\square$
\end{proof}

\begin{proof}[Proof of Theorem 3.2]
Fix $\varepsilon\in(0,1)$ and define
\[
R^{\mathrm{in}}_\mathcal{S}(\lambda)
=
\mathbb P(\hat y\neq y \mid \Delta_\mathcal{S}(x)\ge \lambda,\ y\in \mathcal{S}(x)).
\]
For each $\lambda$, let $N_{\mathrm{in}}(\lambda)$ be the number of
calibration samples in
\(
\mathcal D_{\mathrm{in}}
=
\{(x_i,y_i)\in\mathcal D:\ y_i\in \mathcal{S}(x_i)\}
\)
satisfying $\Delta_\mathcal{S}(x_i)\ge \lambda$,
and let $Y_{\mathrm{in}}(\lambda)$ be the number of errors among them.
Conditional on $N_{\mathrm{in}}(\lambda)$,
\[
Y_{\mathrm{in}}(\lambda)
\sim
\mathrm{Bin}\big(N_{\mathrm{in}}(\lambda), R^{\mathrm{in}}_\mathcal{S}(\lambda)\big).
\]
By exact binomial inversion, the upper confidence bound in \eqref{eq:Rin-ucb-fixedseq} satisfies,
for each fixed $\lambda$,
\[
\mathbb P\!\left(
R^{\mathrm{in}}_\mathcal{S}(\lambda) >
\widehat R^{\mathrm{in},+}_\mathcal{S}(\lambda)
\right)
\le \delta.
\]
Applying the fixed-sequence testing rule, we test from the largest value of $\lambda$ (e.g., 0.999) to a progressively smaller value, and stop at the last time $R_{\mathcal{S}}^{\mathrm{in},+}(\lambda)$ is below the target risk $\varepsilon$, then we get last time $\lambda^*$.
Since $\widehat R^{\mathrm{in},+}_\mathcal{S}(\lambda^*)\le \varepsilon$ is equivalent to rejecting
$H_\lambda: R^{\mathrm{in}}_\mathcal{S}(\lambda^*)>\varepsilon$ at level $\delta$,
the fixed-sequence testing guarantee  implies that the probability
of rejecting any true null hypothesis is at most $\delta$, and therefore
\[
\mathbb P\!\left(
R^{\mathrm{in}}_\mathcal{S}(\lambda^*)>\varepsilon
\right)
\le \delta.
\]
Therefore, with probability at least $1-\delta$ over the calibration sample,
\[
\mathbb P(\hat y\neq y
\mid
\Delta_\mathcal{S}(x)\ge \lambda^*,\ y\in \mathcal{S}(x))
\le \varepsilon,
\]
which completes the proof.
\hfill $\square$
\end{proof}

\begin{proof}[Proof of Corollary 3.3]
Let $A:=\{\Delta_\mathcal{S}(x)\ge \lambda^*\}$. By the law of total probability,
\begin{align*}
\mathbb P(\hat y = y \mid A)
&= \mathbb P(\hat y = y,\, y\in \mathcal{S}(x)\mid A) \\
&= \mathbb P(\hat y = y \mid A,\, y\in \mathcal{S}(x))\\
&\quad\cdot \mathbb P(y\in \mathcal{S}(x)\mid A).
\end{align*}
By \Cref{thm:within-area-fixedseq}, with probability at least $1-\delta$ over the calibration sample used in Stage~II,
\[
\mathbb P(\hat y \neq y \mid A,\, y\in \mathcal{S}(x)) \le \varepsilon,
\]
hence $\mathbb P(\hat y = y \mid A,\, y\in \mathcal{S}(x)) \ge 1-\varepsilon$.

It remains to lower bound $\mathbb P(y\in \mathcal{S}(x)\mid A)$. By the monotonicity assumption that
$\lambda \mapsto \mathbb P(y\notin \mathcal{S}(x)\mid \Delta_\mathcal{S}(x)\ge \lambda)$ is nonincreasing for $\lambda\ge 0$,
we have
\[
\mathbb P(y\notin \mathcal{S}(x)\mid A)
\le
\mathbb P(y\notin \mathcal{S}(x)\mid \Delta_\mathcal{S}(x)\ge 0).
\]
Since $\Delta_\mathcal{S}(x)\ge 0$ almost surely in our setting, the event $\{\Delta_\mathcal{S}(x)\ge 0\}$ is the entire sample
space, and thus
\[
\mathbb P(y\notin \mathcal{S}(x)\mid \Delta_\mathcal{S}(x)\ge 0)=\mathbb P(y\notin \mathcal{S}(x)).
\]
By \Cref{thm:area-c}, $\mathbb P(y\notin \mathcal{S}(x))\le \alpha$, hence
\[
\mathbb P(y\in \mathcal{S}(x)\mid A)=1-\mathbb P(y\notin \mathcal{S}(x)\mid A)\ge 1-\alpha.
\]
Combining the two bounds yields, with probability at least $1-\delta$,
\begin{align*}
&\mathbb P(\hat y = y \mid \Delta_\mathcal{S}(x)\ge \lambda^*)\\
&\quad\quad\quad\quad\quad=
\mathbb P(\hat y = y \mid A)
\ge (1-\varepsilon)(1-\alpha), 
\end{align*}
which proves the corollary.
\hfill $\square$
\end{proof}

\section{Multi-Candidate Data Construction}
\label{app:data-construction}

Each of our four evaluation benchmarks originally provides pairwise human-preference annotations.
To construct multi-candidate instances with $m \in \{5, 10, 20\}$ candidates per query, we follow a general three-step procedure.

\paragraph{Candidate pool construction.}
For each dataset, we group all pairwise-comparison records by query and collect every distinct candidate response, forming a base candidate pool.
We retain the top 3{,}000 queries ranked by the number of available distinct candidates.
When the base pool for a query contains fewer candidates than the largest target $m$, we augment it with additional responses generated by a separate set of open-weight LLMs that do not overlap with the judge models used in our main experiments, so as to avoid any interaction between candidate generation and evaluation.
To prevent the generated candidates from exceeding the quality of the human-preferred response, each generation prompt uses a non-preferred candidate from the base pool as a stylistic reference rather than the ground-truth best response.
After generation, exact-string deduplication is applied to ensure all candidates in the pool are distinct.

\paragraph{Ground-truth assignment.}
For each query, the ground-truth preferred response is identified as the candidate most frequently preferred by human annotators across all available pairwise comparisons.

\paragraph{Instance sampling.}
For each target $m$, we include the ground-truth preferred response together with $m{-}1$ randomly sampled alternatives from the candidate pool, shuffle the ordering, and record the position of the preferred response as the label $y$.
When the candidate pool is sufficiently large relative to $m$, multiple distinct $m$-candidate instances can be drawn from the same query by sampling different candidate subsets.
This yields approximately 3{,}000 instances per (dataset, $m$) pair across all four benchmarks.
All judge models are evaluated on the identical set of instances per dataset, enabling post-hoc composition of cascade configurations.

\begin{figure*}[h]
    \centering
    \includegraphics[width=\textwidth]{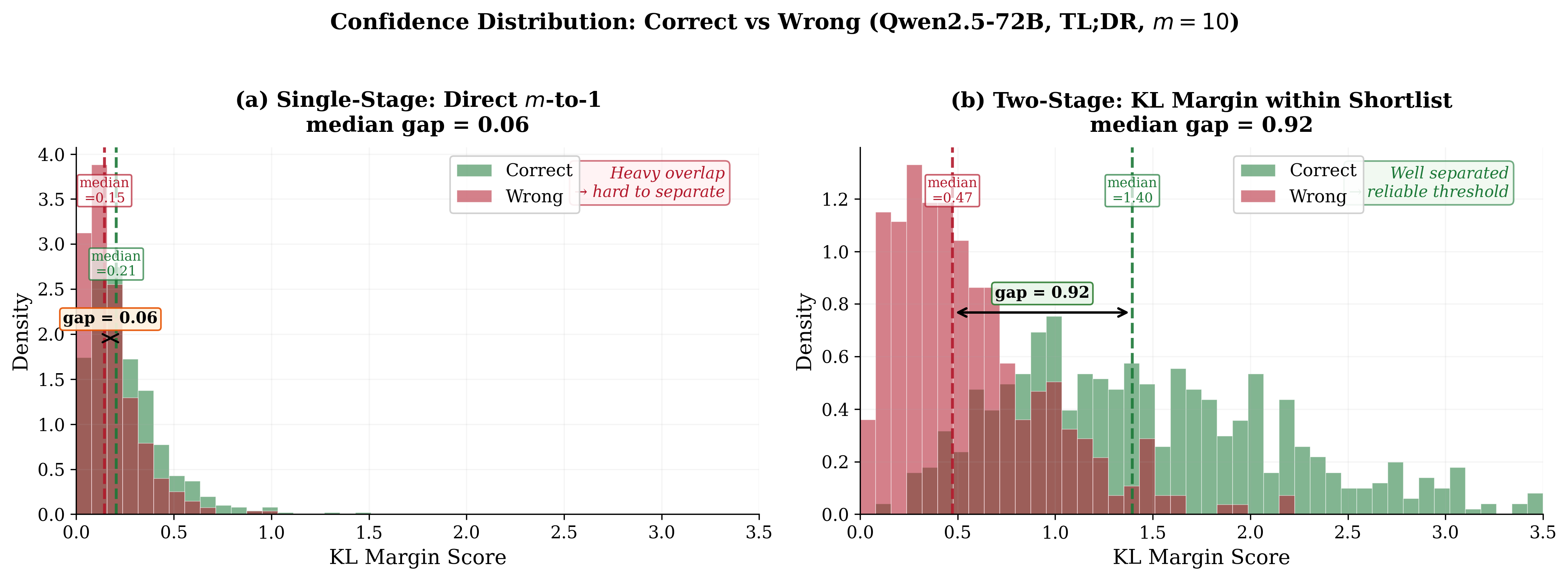}
    \caption{\textbf{Confidence distributions for correct vs.\ wrong predictions} (Qwen2.5-72B, TL;DR, $m{=}10$).
    \textbf{(a)}~Single-stage direct $m$-to-1 KL margin: the correct and wrong distributions overlap heavily (median gap $= 0.06$), making threshold-based selection unreliable.
    \textbf{(b)}~Two-stage KL margin within the shortlist: after Stage~I localization, the distributions are well separated (median gap $= 0.92$), confirming that localization improves the discriminative power of Stage~II scoring.
    }
    \label{fig:conf-dist}
\end{figure*}

\begin{table*}[h]
\centering
\scriptsize
\resizebox{\textwidth}{!}{%
\begin{tabular}{cl ccc ccc ccc}
\toprule
\multirow{2.5}{*}{Dataset} & \multirow{2.5}{*}{Judge} & \multicolumn{3}{c}{Direct $m$-to-$1$} & \multicolumn{3}{c}{Stage~I ($m$-to-$k$)} & \multicolumn{3}{c}{Stage~II ($k$-to-$1$)} \\
\cmidrule(lr){3-5}\cmidrule(lr){6-8}\cmidrule(lr){9-11}
& & $m{=}5$ & $m{=}10$ & $m{=}20$ & $m{=}5$ & $m{=}10$ & $m{=}20$ & $m{=}5$ & $m{=}10$ & $m{=}20$ \\
\midrule
\multirow{3}{*}{TL;DR}
& Mistral-7B    & 0.22 & 0.33 & 0.44 & 0.03 & 0.05 & 0.06 & 0.02 & 0.03 & 0.05 \\
& Qwen2.5-32B   & 0.17 & 0.28 & 0.38 & 0.02 & 0.03 & 0.05 & 0.01 & 0.02 & 0.03 \\
& Llama-3-70B   & 0.15 & 0.25 & 0.35 & 0.02 & 0.03 & 0.04 & 0.01 & 0.02 & 0.03 \\
\midrule
\multirow{3}{*}{\shortstack[l]{Chatbot\\Arena}}
& Mistral-7B    & 0.25 & 0.36 & 0.47 & 0.04 & 0.06 & 0.08 & 0.03 & 0.05 & 0.06 \\
& Qwen2.5-32B   & 0.19 & 0.30 & 0.42 & 0.03 & 0.04 & 0.06 & 0.02 & 0.03 & 0.05 \\
& Llama-3-70B   & 0.16 & 0.27 & 0.38 & 0.02 & 0.03 & 0.05 & 0.01 & 0.02 & 0.04 \\
\midrule
\multirow{3}{*}{HH-RLHF}
& Mistral-7B    & 0.20 & 0.31 & 0.41 & 0.03 & 0.05 & 0.07 & 0.02 & 0.04 & 0.05 \\
& Qwen2.5-32B   & 0.15 & 0.25 & 0.36 & 0.02 & 0.03 & 0.05 & 0.01 & 0.02 & 0.04 \\
& Llama-3-70B   & 0.14 & 0.22 & 0.32 & 0.01 & 0.02 & 0.04 & 0.01 & 0.01 & 0.03 \\
\midrule
\multirow{3}{*}{AlpacaEval}
& Mistral-7B    & 0.23 & 0.35 & 0.46 & 0.04 & 0.06 & 0.07 & 0.03 & 0.04 & 0.06 \\
& Qwen2.5-32B   & 0.18 & 0.29 & 0.40 & 0.02 & 0.04 & 0.05 & 0.01 & 0.03 & 0.04 \\
& Llama-3-70B   & 0.15 & 0.24 & 0.36 & 0.02 & 0.03 & 0.04 & 0.01 & 0.02 & 0.03 \\
\bottomrule
\end{tabular}%
}
\caption{\textbf{Ranking loss (additional judges)} ($T{=}20$ thresholds, $\alpha{=}0.10$). Same format as \Cref{tab:ranking-loss}.}
\label{tab:ranking-loss-app}
\end{table*}

\begin{figure*}[h]
\begin{minipage}[t]{0.49\textwidth}
\begin{algorithm}[H]
\caption{Two-Stage Certified Selection (Ours)}
\label{alg:two-stage}
\small
\begin{algorithmic}[1]
\Require Calibration set $\mathcal{D}=\{(x_i,y_i)\}_{i=1}^n$; test input $x$ with $m$ candidates; risk levels $\alpha,\varepsilon,\delta$; threshold grid $\Lambda=\{0.99,\cdots,0.00\}$
\Require Stage~I scoring $\mathbb{S}_c$ {\small\textit{(default: EMP)}}; Stage~II scoring $\mathbb{S}_h$ {\small\textit{(plug-in: EMP, KL, or Vote)}}; confidence type {\small\textit{(top-1 or margin)}}
\Statex \hrulefill\ \textbf{Stage~I: shortlist size $k$}\ \hrulefill
\For{$i=1,\dots,n$}
  \State $R(z_i) \gets |\{j : \mathbb{S}_c(x_i, g_j) \ge \mathbb{S}_c(x_i, g_{y_i})\}|$
  \Comment{rank of true label}
  \hfill{\color{gray}\eqref{eq:rank-nonconf-c}}
\EndFor
\State $k \gets \lceil(n{+}1)(1{-}\alpha)\rceil$-th order statistic of $\{R(z_i)\}_{i=1}^n$
\hfill{\color{gray}\eqref{eq:k-def-c}}
\Statex \hrulefill\ \textbf{Stage~II: threshold $\lambda^*$}\ \hrulefill
\State $\mathcal{S}(x_i)\!\gets\!\text{top-}k$ under $\mathbb{S}_c(x_i,\cdot)$, $\forall\,i$
\hfill{\color{gray}\eqref{eq:area-set-c}}
\State $\mathcal{D}_{\mathrm{in}}\!\gets\!\{(x_i,y_i)\!\in\!\mathcal{D}: y_i\!\in\!\mathcal{S}(x_i)\}$
\For{$j\in\mathcal{D}_{\mathrm{in}}$}
  \State $\hat y_j\!\gets\!\arg\max_{i\in\mathcal{S}(x_j)}\mathbb{S}_h(x_j,g_i)$
  \State $c_j\!\gets\!\textsc{Conf}(\hat y_j,\,\mathcal{S}(x_j),\,\mathbb{S}_h)$
  \hfill{\color{gray}Alg.\,\ref{alg:conf}}
\EndFor
\State $\lambda^*\gets -\infty$
\For{$\lambda\in\Lambda$ \textbf{(largest\,$\to$\,smallest)}}
  \State $N^{\mathrm{in}}\!\gets\!|\{j\!\in\!\mathcal{D}_{\mathrm{in}}: c_j\!\ge\!\lambda\}|$
  \State $Y^{\mathrm{in}}\!\gets\!\!\sum_{j\in\mathcal{D}_{\mathrm{in}},\,c_j\ge\lambda}\!\mathbf{1}[\hat y_j\!\ne\! y_j]$
  \State $\widehat{R}^{\mathrm{in},+}\!\gets\!\mathrm{BinomUCB}_\delta(Y^{\mathrm{in}},N^{\mathrm{in}})$
  \hfill{\color{gray}\eqref{eq:Rin-ucb-fixedseq}}
  \State \textbf{if} $\widehat{R}^{\mathrm{in},+}\!\le\!\varepsilon$
    \textbf{then} $\lambda^*\!\gets\!\lambda$;\
    \textbf{else break}
  \hfill{\color{gray}\eqref{eq:lambda-star-fixedseq}}
\EndFor
\Statex \hrulefill\ \textbf{Inference (new $x$)}\ \hrulefill
\State $\mathcal{S}(x) \gets$ indices of top-$k$ candidates under $\mathbb{S}_c(x,\cdot)$ \Comment{Stage~I: localize}
\State $\hat{y} \gets \arg\max_{i\in\mathcal{S}(x)} \mathbb{S}_h(x,g_i)$ \Comment{Stage~II: predict}
\State $c(x)\gets\textsc{Conf}(\hat y,\,\mathcal{S}(x),\,\mathbb{S}_h)$
\hfill{\color{gray}Alg.\,\ref{alg:conf}}
\State \Return $\hat y$ if $c(x)\!\ge\!\lambda^*$;\ \textsc{Abstain} o.w.
\end{algorithmic}
\end{algorithm}
\end{minipage}%
\hfill
\begin{minipage}[t]{0.49\textwidth}
\begin{algorithm}[H]
\caption{Single-Stage Baseline~\citep{jung2024trust}}
\label{alg:single-stage}
\small
\begin{algorithmic}[1]
\Require Calibration set $\mathcal{D}=\{(x_i,y_i)\}_{i=1}^n$; test input $x$ with $m$ candidates; risk level $\varepsilon',\delta$; threshold grid $\Lambda=\{0.99,\cdots,0.00\}$
\Require Per-candidate scoring $\mathbb{S}_h$ {\small\textit{(plug-in: EMP, KL, or Vote)}}; confidence type {\small\textit{(top-1 or margin)}}
\Statex \hrulefill\ \textbf{Calibration (no localization)}\ \hrulefill
\State $C\gets\{1,\dots,m\}$
\For{$j\in\mathcal{D}$}
  \State $\hat y_j\gets\arg\max_{i\in C}\mathbb{S}_h(x_j,g_i)$
  \State $c_j\gets\textsc{Conf}(\hat y_j,\,C,\,\mathbb{S}_h)$
  \hfill{\color{gray}Alg.\,\ref{alg:conf}}
\EndFor
\State $\lambda^*\gets -\infty$
\For{$\lambda\in\Lambda$ \textbf{(largest\,$\to$\,smallest)}}
  \State $N\gets|\{j: c_j\ge\lambda\}|$
  \State $Y\gets\sum_{j:\,c_j\ge\lambda}\mathbf{1}[\hat y_j\ne y_j]$
  \State $\widehat R^+\!\gets\!\mathrm{BinomUCB}_\delta(Y,N)$
  \State \textbf{if} $\widehat R^+\!\le\!\varepsilon'$
    \textbf{then} $\lambda^*\!\gets\!\lambda$;\
    \textbf{else break}
\EndFor
\Statex \hrulefill\ \textbf{Inference (new $x$)}\ \hrulefill
\State $\hat y\gets\arg\max_{i\in C}\mathbb{S}_h(x,g_i)$
\State $c(x)\gets\textsc{Conf}(\hat y,\,C,\,\mathbb{S}_h)$
\State \Return $\hat y$ if $c(x)\!\ge\!\lambda^*$;\ \textsc{Abstain} o.w.
\end{algorithmic}
\end{algorithm}
\vspace{-1.0em}
\begin{algorithm}[H]
\caption{\textsc{Conf}: Confidence Function}
\label{alg:conf}
\small
\begin{algorithmic}[1]
\Require Winner $\hat y$; candidate set $C$; scoring $\mathbb{S}_h$
\State $s_1\gets\mathbb{S}_h(x,g_{\hat y})$
\State $s_2\gets\max_{i\in C,\,i\ne\hat y}\mathbb{S}_h(x,g_i)$
\If{\textbf{top-1} mode}
  \State \Return $s_1$
\ElsIf{\textbf{margin} mode} \hfill{\color{gray}\eqref{eq:marginS}}
  \State \Return $s_1 - s_2$
\EndIf
\end{algorithmic}
\end{algorithm}
\end{minipage}
\end{figure*}

\section{Scoring Functions}
\label{app:scoring}

This section defines all scoring functions used in our experiments.
All scores are computed from the stored Simulated Annotator probability vectors~\citep{jung2024trust}: for each instance $x$ with $m$ candidates, we collect $\mathbf{p}_j = (p_{j,1},\dots,p_{j,m})$ from each of $N{=}5$ annotators (each conditioned on $K{=}5$ few-shot examples), where $p_{j,i} = P_\theta(y{=}i \mid x; \text{annotator}_j)$.

\subsection{Stage~I: Conformal Localization Scoring}
\label{app:scoring-stage1}

Stage~I uses a per-candidate score $\mathbb{S}_c(x, g_i)$ to rank all $m$ candidates for conformal localization.
Throughout this paper we use the EMP:
\[
\mathbb{S}_c(x, g_i) \;=\; \bar{p}_i \;=\; \frac{1}{N}\sum_{j=1}^{N} p_{j,i},
\]
the average predicted probability across annotators.
The top-$k$ candidates under EMP form the shortlist $\mathcal{S}(x)$.
This scoring function is used exclusively for Stage~I localization and is fixed throughout all experiments.

\subsection{Stage~II and Single-Stage: Selection Scoring}
\label{app:scoring-stage2}

Stage~II of our two-stage framework and the single-stage baseline of \citet{jung2024trust} share the same set of scoring functions; the only difference is the candidate set over which they operate:
Stage~II scores candidates within the localized shortlist $\mathcal{S}(x)$, while the single-stage baseline scores all $m$ candidates directly.
In both cases, the system uses a per-candidate scoring function $\mathbb{S}_h$ to rank candidates, then applies a confidence function to decide whether to accept or abstain.

\paragraph{Per-candidate scoring functions.}
Given a candidate set $C$ (the shortlist $\mathcal{S}(x)$ for Stage~II, or the full set $\{1,\dots,m\}$ for the single-stage baseline), a per-candidate scoring function assigns a score $\mathbb{S}_h(x, g_i)$ to each $i \in C$.
The predicted winner is $\hat{y} = \arg\max_{i \in C} \mathbb{S}_h(x, g_i)$.
We consider three choices:

\begin{itemize}[nosep,leftmargin=*]
\item \textbf{EMP.}
$\mathbb{S}_h(x, g_i) = \bar{p}_i = \frac{1}{N}\sum_{j=1}^{N} p_{j,i}$.
This uses the same ensemble mean probability as Stage~I but serves a different role: here it provides per-candidate scores for the selection decision within $C$, rather than for conformal localization over all $m$ candidates.

\item \textbf{KL-divergence scoring with variance normalization.}
We derive the per-candidate score from KL divergence.
First, we renormalize each annotator's probabilities to the candidate set $C$:
$q_{i}^{(j)} = p_{j,i} / \sum_{l \in C} p_{j,l}$,
which yields a valid distribution over $C$ and removes the influence of eliminated candidates.
To measure how strongly annotator $j$ supports candidate $i$, we compute the KL divergence from the point mass $e_i$ (i.e., the distribution that places all mass on $i$) to the annotator's renormalized distribution:
$\mathrm{KL}(e_i \| \mathbf{q}^{(j)}) = -\log q_i^{(j)}$.
A smaller KL divergence indicates stronger support.
Since KL is a distance (smaller is better) but we need a weight (larger is better), we apply the exponential transform:
$\exp(-\mathrm{KL}(e_i \| \mathbf{q}^{(j)})) = q_i^{(j)}$.
That is, the renormalized probability $q_i^{(j)}$ directly serves as the KL-derived voting weight.
The per-candidate score aggregates this weight across annotators who vote for $i$ and normalizes for cross-instance comparability:
\begin{small}
\begin{equation}
\mathbb{S}_h(x, g_i) = \frac{\frac{1}{N}\sum_{j=1}^{N} q_i^{(j)} \cdot \mathbf{1}\!\big\{\arg\max_{l \in C} q_l^{(j)} = i\big\}}{\sqrt{\sigma_C^2 + \bar{v}_i}},
\label{eq:kl-score}
\end{equation}
\end{small}where the numerator averages the KL-derived weight $q_i^{(j)}$ over the annotators whose top choice within $C$ is $i$;
$\bar{v}_i$ is the variance of $q_i^{(j)}$ among those voters (rewarding consensus); and
$\sigma_C^2$ is the variance of $\{\bar{q}_i\}_{i \in C}$ across candidates (a shared scale factor that enables cross-instance comparability and prevents the denominator from collapsing when $\bar{v}_i$ is small).

\item \textbf{Vote.}
The vote share for candidate~$i$:
$\mathbb{S}_h(x, g_i) = \frac{1}{N}\sum_{j=1}^{N} \mathbf{1}\!\big[\arg\max_{i' \in C} p_{j,i'} = i\big]$,
i.e., the fraction of annotators whose top-1 prediction within $C$ is $i$.
This is a non-parametric alternative that discards probability magnitudes and only counts consensus.
\end{itemize}

\paragraph{Confidence functions.}
Given per-candidate scores $\mathbb{S}_h(x, g_i)$ and the predicted winner $\hat{y}$, a confidence function produces a scalar that is compared against the calibrated threshold $\lambda^*$ to decide acceptance.
We consider two types:

\begin{itemize}[nosep,leftmargin=*]
\item \textbf{Top-1.}
The confidence is the score of the predicted winner:
$c(x) = \mathbb{S}_h(x, g_{\hat y})$.
The system accepts if $c(x) \ge \lambda^*$.
This captures whether the top candidate has a sufficiently high absolute score.

\item \textbf{Margin.}
The confidence is the gap between the top two candidates:
$c(x) = \mathbb{S}_h(x, g_{\hat y}) - \max_{i \in C,\, i \neq \hat y} \mathbb{S}_h(x, g_i)$.
The system accepts if $c(x) \ge \lambda^*$.
This captures whether the top candidate is sufficiently separated from the runner-up, corresponding to $\Delta_\mathcal{S}(x)$ in \eqref{eq:marginS}.
\end{itemize}

\paragraph{Named configurations.}
Combining three per-candidate scoring functions with two confidence functions yields the following named configurations, used consistently in both the single-stage and two-stage columns of \Cref{tab:main,tab:main-app}:

\begin{center}
\small
\resizebox{\columnwidth}{!}{%
\begin{tabular}{lll}
\toprule
Name in tables & Per-candidate scoring & Confidence type \\
\midrule
top-1 EMP    & EMP  & top-1 \\
EMP margin   & EMP  & margin \\
top-1 KL     & KL-divergence scoring & top-1 \\
KL margin    & KL-divergence scoring & margin \\
Vote         & Vote & top-1 \\
\bottomrule
\end{tabular}%
}
\end{center}

\begin{figure*}[h]
    \centering
    \includegraphics[width=\textwidth]{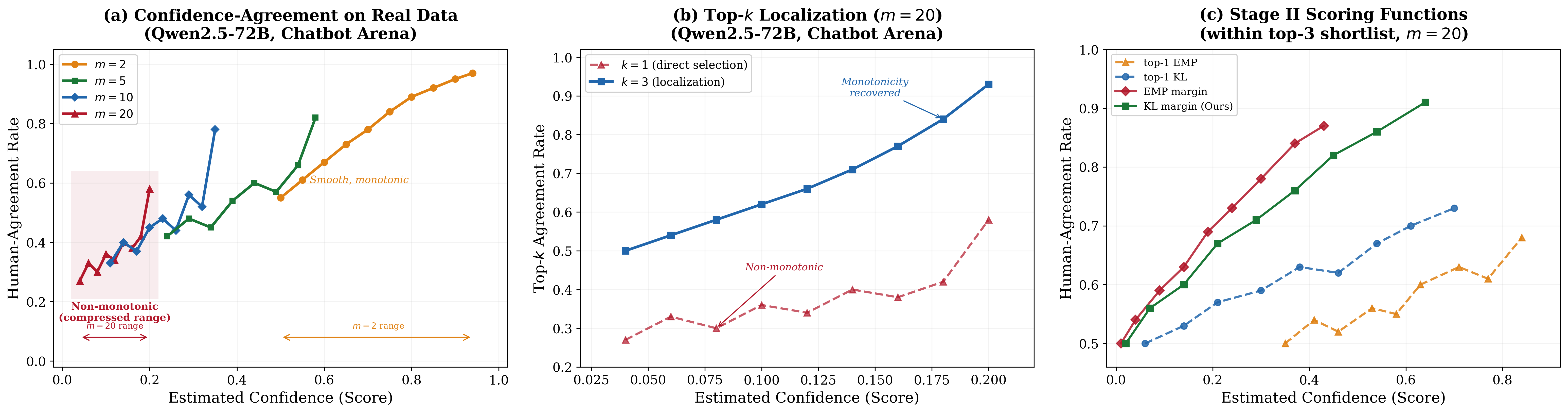}
    \vspace{-8mm}
    \caption{\textbf{Confidence--agreement curves on Chatbot Arena and Qwen2.5-72B.}
    Same layout as \Cref{fig:real-motivation}.
    \textbf{(a)}~Single-stage $m$-to-$1$.
    \textbf{(b)}~Stage~I top-$k$ localization.
    \textbf{(c)}~Stage~II scoring within the top-3 shortlist ($m{=}20$).
    }
    \label{fig:real-motivation-chatbot}
\end{figure*}

\begin{figure*}[h]
    \centering
    \includegraphics[width=\textwidth]{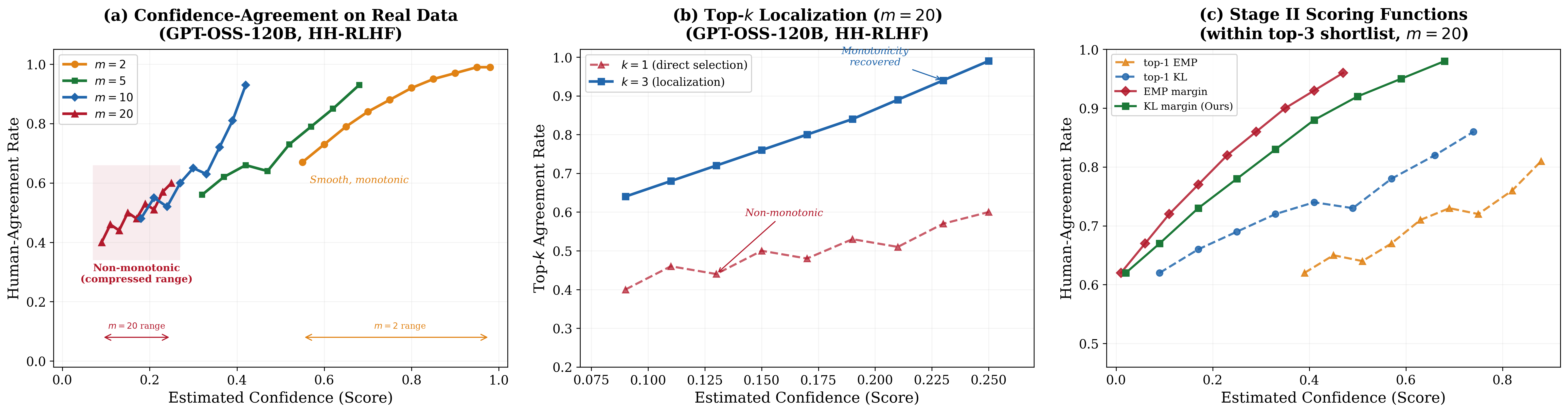}
    \vspace{-8mm}
    \caption{\textbf{Confidence--agreement curves on HH-RLHF and GPT-OSS-120B.}
    Same layout as \Cref{fig:real-motivation}.
    \textbf{(a)}~Single-stage $m$-to-$1$.
    \textbf{(b)}~Stage~I top-$k$ localization.
    \textbf{(c)}~Stage~II scoring within the top-3 shortlist ($m{=}20$).
    }
    \label{fig:real-motivation-hh}
\end{figure*}

\section{Algorithm Pseudocode}
\label{app:pseudocode}

\Cref{alg:two-stage} presents the full pseudocode for our two-stage certified selection framework.
The scoring functions $\mathbb{S}_c$ (Stage~I) and $\mathbb{S}_h$ (Stage~II) are plug-in components: $\mathbb{S}_c$ is fixed to EMP throughout, while $\mathbb{S}_h$ can be any of the per-candidate scoring functions defined in \Cref{app:scoring-stage2} (EMP, KL, or Vote), combined with either a top-1 or margin confidence function.
\Cref{alg:single-stage} shows the single-stage baseline of \citet{jung2024trust}, which corresponds to applying Stage~II directly to all $m$ candidates without prior localization.
Both algorithms are structured in two phases: calibration (offline, on $\mathcal{D}$) and test-time inference (online, per new $x$).
The key structural difference is that \Cref{alg:two-stage} first determines the shortlist size $k$ via conformal calibration and then calibrates $\lambda^*$ only on the covered subset $\mathcal{D}_{\mathrm{in}}$, whereas \Cref{alg:single-stage} calibrates $\lambda^*$ directly on all $m$ candidates without localization.

\begin{table*}[h]
\centering
\footnotesize
\resizebox{\textwidth}{!}{%
\begin{tabular}{cl ccc ccc ccc ccc}
\toprule
\multirow{2.5}{*}{Judge} & \multirow{2.5}{*}{Method} & \multicolumn{3}{c}{TL;DR} & \multicolumn{3}{c}{Chatbot Arena} & \multicolumn{3}{c}{HH-RLHF} & \multicolumn{3}{c}{AlpacaEval} \\
\cmidrule(lr){3-5}\cmidrule(lr){6-8}\cmidrule(lr){9-11}\cmidrule(lr){12-14}
& & Cov. & GSR & Agr. & Cov. & GSR & Agr. & Cov. & GSR & Agr. & Cov. & GSR & Agr. \\
\midrule
\multirow{8}{*}{\shortstack[l]{Mistral\\7B}}
& No Selection                   & 100\%& 0\% & 35\%\std{1.3} & 100\%& 0\% & 30\%\std{1.5} & 100\%& 0\% & 39\%\std{1.1} & 100\%& 0\% & 33\%\std{1.4} \\
& Single-Stage (EMP)                   & 25\%\std{3.3} & 52.4\% & 77\%\std{2.3} & 20\%\std{3.7} & 46.8\% & 74\%\std{2.5} & 29\%\std{3.0} & 57.6\% & 79\%\std{2.0} & 23\%\std{3.5} & 50.3\% & 76\%\std{2.3} \\
& Single-Stage (EMP margin)            & 30\%\std{3.1} & 57.3\% & 78\%\std{2.1} & 25\%\std{3.4} & 51.8\% & 76\%\std{2.3} & 34\%\std{2.8} & 62.7\% & 80\%\std{1.9} & 28\%\std{3.2} & 55.1\% & 77\%\std{2.2} \\
& Single-Stage (KL margin)              & 34\%\std{2.9} & 62.4\% & 80\%\std{2.0} & 29\%\std{3.2} & 56.9\% & 78\%\std{2.1} & 38\%\std{2.6} & 67.8\% & 82\%\std{1.8} & 32\%\std{3.0} & 60.6\% & 79\%\std{2.1} \\
& Single-Stage (Vote)                   & 32\%\std{3.0} & 59.8\% & 79\%\std{2.0} & 27\%\std{3.3} & 54.1\% & 77\%\std{2.2} & 36\%\std{2.7} & 64.5\% & 81\%\std{1.8} & 30\%\std{3.1} & 57.9\% & 78\%\std{2.1} \\
\cmidrule(l){2-14}
\rowcolor{blue!6} \cellcolor{white} & \textbf{Two-Stage (EMP margin)}     & 38\%\std{2.7} & \textbf{90.6\%} & 84\%\std{1.6} & 33\%\std{3.0} & \textbf{90.1\%} & 82\%\std{1.8} & 42\%\std{2.4} & \textbf{92.3\%} & 85\%\std{1.4} & 36\%\std{2.8} & \textbf{90.4\%} & 83\%\std{1.7} \\
\rowcolor{blue!6} \cellcolor{white} & \textbf{Two-Stage (KL margin)}       & 44\%\std{2.4} & \textbf{94.7\%} & 86\%\std{1.4} & 39\%\std{2.7} & \textbf{93.2\%} & 85\%\std{1.6} & 48\%\std{2.2} & \textbf{95.6\%} & 87\%\std{1.2} & 42\%\std{2.5} & \textbf{93.8\%} & 86\%\std{1.5} \\
\rowcolor{blue!6} \cellcolor{white} & \textbf{Two-Stage (Vote)}            & 41\%\std{2.5} & \textbf{92.8\%} & 85\%\std{1.5} & 36\%\std{2.8} & \textbf{91.4\%} & 83\%\std{1.7} & 45\%\std{2.3} & \textbf{93.7\%} & 86\%\std{1.3} & 39\%\std{2.6} & \textbf{92.1\%} & 85\%\std{1.6} \\
\midrule
\multirow{8}{*}{\shortstack[l]{Qwen2.5\\32B}}
& No Selection                   & 100\%& 0\% & 44\%\std{1.2} & 100\%& 0\% & 39\%\std{1.3} & 100\%& 0\% & 47\%\std{1.0} & 100\%& 0\% & 43\%\std{1.2} \\
& Single-Stage (EMP)                   & 34\%\std{2.8} & 57.6\% & 79\%\std{2.0} & 29\%\std{3.2} & 52.3\% & 77\%\std{2.2} & 38\%\std{2.6} & 61.8\% & 81\%\std{1.8} & 33\%\std{2.9} & 56.4\% & 78\%\std{2.1} \\
& Single-Stage (EMP margin)            & 39\%\std{2.6} & 62.8\% & 80\%\std{1.9} & 34\%\std{2.9} & 57.4\% & 79\%\std{2.1} & 43\%\std{2.5} & 66.9\% & 82\%\std{1.7} & 38\%\std{2.7} & 61.3\% & 80\%\std{1.9} \\
& Single-Stage (KL margin)              & 43\%\std{2.5} & 67.6\% & 82\%\std{1.7} & 38\%\std{2.8} & 62.8\% & 80\%\std{1.9} & 47\%\std{2.3} & 71.4\% & 83\%\std{1.6} & 42\%\std{2.6} & 66.1\% & 82\%\std{1.8} \\
& Single-Stage (Vote)                   & 41\%\std{2.6} & 64.7\% & 81\%\std{1.8} & 36\%\std{2.8} & 59.6\% & 79\%\std{2.0} & 45\%\std{2.4} & 68.3\% & 82\%\std{1.6} & 40\%\std{2.6} & 63.4\% & 81\%\std{1.8} \\
\cmidrule(l){2-14}
\rowcolor{blue!6} \cellcolor{white} & \textbf{Two-Stage (EMP margin)}     & 49\%\std{2.3} & \textbf{91.4\%} & 85\%\std{1.4} & 44\%\std{2.5} & \textbf{90.3\%} & 84\%\std{1.6} & 53\%\std{2.1} & \textbf{92.8\%} & 86\%\std{1.2} & 48\%\std{2.4} & \textbf{91.1\%} & 85\%\std{1.5} \\
\rowcolor{blue!6} \cellcolor{white} & \textbf{Two-Stage (KL margin)}       & 55\%\std{2.1} & \textbf{95.2\%} & 87\%\std{1.2} & 50\%\std{2.3} & \textbf{93.8\%} & 86\%\std{1.4} & 59\%\std{1.9} & \textbf{95.7\%} & 88\%\std{1.0} & 54\%\std{2.1} & \textbf{94.6\%} & 87\%\std{1.3} \\
\rowcolor{blue!6} \cellcolor{white} & \textbf{Two-Stage (Vote)}            & 52\%\std{2.2} & \textbf{93.6\%} & 86\%\std{1.3} & 47\%\std{2.4} & \textbf{92.1\%} & 85\%\std{1.5} & 56\%\std{2.0} & \textbf{93.9\%} & 87\%\std{1.1} & 51\%\std{2.2} & \textbf{92.7\%} & 86\%\std{1.4} \\
\midrule
\multirow{8}{*}{\shortstack[l]{Llama-3\\70B}}
& No Selection                   & 100\%& 0\% & 50\%\std{1.1} & 100\%& 0\% & 45\%\std{1.3} & 100\%& 0\% & 53\%\std{1.0} & 100\%& 0\% & 49\%\std{1.2} \\
& Single-Stage (EMP)                   & 39\%\std{2.7} & 60.6\% & 80\%\std{1.9} & 34\%\std{3.0} & 55.3\% & 78\%\std{2.1} & 43\%\std{2.6} & 64.8\% & 82\%\std{1.7} & 38\%\std{2.8} & 59.7\% & 80\%\std{2.0} \\
& Single-Stage (EMP margin)            & 44\%\std{2.5} & 65.4\% & 81\%\std{1.8} & 39\%\std{2.8} & 60.7\% & 79\%\std{2.0} & 48\%\std{2.4} & 69.6\% & 83\%\std{1.6} & 43\%\std{2.6} & 64.8\% & 81\%\std{1.9} \\
& Single-Stage (KL margin)              & 48\%\std{2.4} & 70.3\% & 83\%\std{1.7} & 43\%\std{2.6} & 65.8\% & 81\%\std{1.9} & 52\%\std{2.2} & 74.6\% & 84\%\std{1.5} & 47\%\std{2.4} & 69.4\% & 83\%\std{1.7} \\
& Single-Stage (Vote)                   & 46\%\std{2.5} & 67.8\% & 82\%\std{1.7} & 41\%\std{2.7} & 62.4\% & 80\%\std{2.0} & 50\%\std{2.3} & 71.9\% & 84\%\std{1.6} & 45\%\std{2.5} & 66.3\% & 82\%\std{1.8} \\
\cmidrule(l){2-14}
\rowcolor{blue!6} \cellcolor{white} & \textbf{Two-Stage (EMP margin)}     & 55\%\std{2.1} & \textbf{92.1\%} & 86\%\std{1.3} & 50\%\std{2.4} & \textbf{90.6\%} & 85\%\std{1.5} & 59\%\std{2.0} & \textbf{93.4\%} & 87\%\std{1.1} & 54\%\std{2.2} & \textbf{91.8\%} & 86\%\std{1.4} \\
\rowcolor{blue!6} \cellcolor{white} & \textbf{Two-Stage (KL margin)}       & 62\%\std{1.9} & \textbf{95.4\%} & 88\%\std{1.1} & 57\%\std{2.1} & \textbf{94.1\%} & 87\%\std{1.3} & 66\%\std{1.8} & \textbf{96.3\%} & 89\%\std{0.9} & 61\%\std{2.0} & \textbf{95.2\%} & 88\%\std{1.2} \\
\rowcolor{blue!6} \cellcolor{white} & \textbf{Two-Stage (Vote)}            & 59\%\std{2.0} & \textbf{93.8\%} & 87\%\std{1.2} & 54\%\std{2.2} & \textbf{92.6\%} & 86\%\std{1.4} & 63\%\std{1.9} & \textbf{94.9\%} & 88\%\std{1.0} & 58\%\std{2.1} & \textbf{93.4\%} & 87\%\std{1.3} \\
\bottomrule
\end{tabular}%
}
\caption{\textbf{Two-stage vs.\ single-stage (additional judges)} ($m{=}10$, target agreement $= 0.81$, $\delta{=}0.10$). Same format as \Cref{tab:main}. Coverage and Agr.\ are reported as mean $\pm$ std over 1{,}000 random splits.}
\label{tab:main-app}
\end{table*}

\begin{table*}[h]
\centering
\scriptsize
\resizebox{\textwidth}{!}{%
\begin{tabular}{l l cc ccccc}
\toprule
\multirow{4}{*}{Dataset} & \multirow{4}{*}{Cascade Configuration} & \multicolumn{2}{c}{Single-Stage} & \multicolumn{5}{c}{\textbf{Two-Stage (Ours)}} \\
\cmidrule(lr){3-4}\cmidrule(lr){5-9}
& & \multirow{2.5}{*}{Cov.} & \multirow{2.5}{*}{GSR} & \multirow{2.5}{*}{Cov.} & \multirow{2.5}{*}{GSR} & \multicolumn{3}{c}{Comp.\ (\%)} \\
\cmidrule(lr){7-9}
& & & & & & $T_1$ & $T_2$ & $T_3$ \\
\midrule
\multirow{4}{*}{\shortstack[l]{Chatbot\\Arena}}
& Qwen2.5-72B only                   & 44\%\std{2.6} & 67.4\%  &\cellcolor{blue!6} 58\%\std{2.2} &\cellcolor{blue!6} 94.6\%  & -- & -- & 100.0 \\
& Qwen-7B $\to$ 32B $\to$ 72B       & 55\%\std{3.0} & 61.8\%  &\cellcolor{blue!6} 74\%\std{2.1} &\cellcolor{blue!6} \textbf{92.3\%}  & 38.5 & 35.2 & 26.3 \\
& Mistral-7B $\to$ DS-67B $\to$ GPT-120B & 53\%\std{3.2} & 60.4\%  &\cellcolor{blue!6} 73\%\std{2.3} &\cellcolor{blue!6} \textbf{91.7\%}  & 35.8 & 37.9 & 26.3 \\
& DeepSeek-16B $\to$ Qwen-72B       & 52\%\std{3.4} & 59.6\%  &\cellcolor{blue!6} 70\%\std{2.5} &\cellcolor{blue!6} \textbf{90.8\%}  & 46.1 & 53.9 & -- \\
\midrule
\multirow{4}{*}{HH-RLHF}
& Qwen2.5-72B only                   & 52\%\std{2.3} & 75.3\%  &\cellcolor{blue!6} 66\%\std{1.8} &\cellcolor{blue!6} 96.2\%  & -- & -- & 100.0 \\
& Qwen-7B $\to$ 32B $\to$ 72B       & 64\%\std{2.5} & 69.4\%  &\cellcolor{blue!6} 81\%\std{1.8} &\cellcolor{blue!6} \textbf{93.6\%}  & 37.1 & 36.0 & 26.9 \\
& Mistral-7B $\to$ DS-67B $\to$ GPT-120B & 62\%\std{2.7} & 67.8\%  &\cellcolor{blue!6} 80\%\std{2.0} &\cellcolor{blue!6} \textbf{92.8\%}  & 34.7 & 38.5 & 26.8 \\
& DeepSeek-16B $\to$ Qwen-72B       & 60\%\std{2.9} & 66.3\%  &\cellcolor{blue!6} 77\%\std{2.2} &\cellcolor{blue!6} \textbf{91.6\%}  & 44.8 & 55.2 & -- \\
\midrule
\multirow{4}{*}{AlpacaEval}
& Qwen2.5-72B only                   & 49\%\std{2.5} & 71.2\%  &\cellcolor{blue!6} 63\%\std{2.0} &\cellcolor{blue!6} 95.1\%  & -- & -- & 100.0 \\
& Qwen-7B $\to$ 32B $\to$ 72B       & 58\%\std{2.8} & 63.6\%  &\cellcolor{blue!6} 76\%\std{2.0} &\cellcolor{blue!6} \textbf{92.4\%}  & 38.0 & 35.8 & 26.2 \\
& Mistral-7B $\to$ DS-67B $\to$ GPT-120B & 56\%\std{3.0} & 62.3\%  &\cellcolor{blue!6} 75\%\std{2.2} &\cellcolor{blue!6} \textbf{91.8\%}  & 35.5 & 38.1 & 26.4 \\
& DeepSeek-16B $\to$ Qwen-72B       & 54\%\std{3.1} & 61.7\%  &\cellcolor{blue!6} 72\%\std{2.3} &\cellcolor{blue!6} \textbf{90.9\%}  & 45.9 & 54.1 & -- \\
\bottomrule
\end{tabular}%
}
\caption{\textbf{Cascaded evaluation on additional datasets} ($m{=}10$, target agreement $= 0.81$, $\delta{=}0.10$). Same format as \Cref{tab:cascade}.}
\label{tab:cascade-app}
\vspace{-4mm}
\end{table*}

\section{Monotonicity Evaluation Settings and Confidence Distribution Analysis}
\label{app:conf-dist}

This appendix provides the detailed experimental settings used to produce the monotonicity results in \Cref{tab:ranking-loss} and \Cref{fig:real-motivation}, and presents additional confidence-distribution analysis.

\paragraph{Confidence and agreement definitions.}
\Cref{tab:ranking-loss} and \Cref{fig:real-motivation} share the same three evaluation settings, each corresponding to a distinct stage of the selection pipeline.
The confidence and agreement definitions for each setting are as follows:

\begin{itemize}[nosep,leftmargin=*]
    \item \textbf{Direct $m$-to-$1$ (single-stage baseline):} The confidence score is the top-1 EMP, i.e., $c(x)=\max_i\bar{p}_i$ where $\bar{p}_i = \frac{1}{N}\sum_{j=1}^{N}p_{j,i}$. The agreement indicator is $\mathbf{1}[\hat y = y]$, where $\hat y = \arg\max_i \bar{p}_i$.
    \item \textbf{Stage~I ($m$-to-$k$ localization):} The confidence score is the same top-1 EMP as above. The agreement indicator changes to $\mathbf{1}[y\in\text{top-}k]$, i.e., whether the human-preferred response lies within the top-$k$ candidates ranked by EMP. The shortlist size $k$ is determined by the conformal quantile (\Cref{eq:k-def-c}) with $\alpha{=}0.10$.
    \item \textbf{Stage~II ($k$-to-$1$ selection within shortlist):} The confidence score is the KL-divergence margin $\Delta_\mathcal{S}(x)$ defined in \Cref{eq:marginS}, computed within the shortlist $\mathcal{S}(x)$. The agreement indicator is $\mathbf{1}[\hat y = y \mid y\in\mathcal{S}(x)]$, restricted to instances where the human-preferred response is covered by the shortlist.
\end{itemize}

\paragraph{Ranking loss computation (\Cref{tab:ranking-loss}).}
For \Cref{tab:ranking-loss}, the $T{=}20$ confidence thresholds are evenly spaced over the observed range $[\min_i c(x_i),\, \max_i c(x_i)]$ for each (dataset, judge, $m$) combination.
The ranking loss (defined in \Cref{sec:exp-motivation}) is computed on the full dataset without calibration/test splitting to maximize statistical stability.

\paragraph{Confidence--agreement curves (\Cref{fig:real-motivation}).}
\Cref{fig:real-motivation} visualizes the confidence--agreement relationship using Qwen2.5-72B on TL;DR as a representative pair, chosen because it is a mid-strength judge that clearly exhibits the monotonicity breakdown in single-stage evaluation.
The three panels correspond to the three settings defined above:
panel~(a) sweeps $m\in\{2,5,10,20\}$ under the Direct $m$-to-$1$ setting;
panel~(b) fixes $m{=}20$ and compares $k{=}1$ (direct selection) against $k{=}3$ (localization);
panel~(c) fixes $m{=}20$ within the top-3 shortlist and compares four Stage~II scoring functions (top-1 EMP, top-1 KL, EMP margin, and KL margin, as defined in \Cref{app:scoring}).
The same qualitative patterns hold for other (dataset, judge) combinations, as confirmed by \Cref{tab:ranking-loss,tab:ranking-loss-app}.
\Cref{fig:real-motivation-chatbot,fig:real-motivation-hh} present the corresponding figures for Chatbot Arena (Qwen2.5-72B) and HH-RLHF (GPT-OSS-120B), showing that the monotonicity breakdown and two-stage recovery are consistent across both datasets and judge models.

\paragraph{Confidence-score distributions (\Cref{fig:conf-dist}).}
\Cref{fig:conf-dist} visualizes the distributions of KL margin scores for correct versus incorrect predictions.
In the single-stage setting (panel~a), the KL margin is computed over all $m$ candidates; in the two-stage setting (panel~b), it is computed within the shortlist $\mathcal{S}(x)$.
After Stage~I localization, the correct and incorrect distributions become substantially more separated (median gap increases from $0.06$ to $0.92$), confirming that localization improves the discriminative power of Stage~II scoring and enables more reliable threshold-based selection.

\section{Additional Validation and Ablation Studies}
\label{app:additional-validation}

This section reports targeted analyses of the cross-stage monotonicity condition, a strengthened single-stage baseline, and the components of the KL-margin score.

\subsection{Direct Validation of Cross-Stage Monotonicity}

For each configuration, we estimate
$\mathbb{P}(y\notin\mathcal{S}(x)\mid\Delta_{\mathcal{S}}(x)\geq\lambda)$
over 1{,}000 random calibration/test splits with $\alpha{=}0.10$.
\Cref{tab:cross-stage-monotonicity} reports representative configurations.

\begin{table*}[t]
\centering
\scriptsize
\resizebox{\textwidth}{!}{%
\begin{tabular}{lccccc}
\toprule
(Judge, Dataset, $m$) & $\lambda{=}0.0$ & $\lambda{=}0.2$ & $\lambda{=}0.4$ & $\lambda{=}0.6$ & $\lambda{=}0.8$ \\
\midrule
Llama-3-8B, TL;DR, $m{=}5$              & 9.6\% & 5.8\% & 3.1\% & 1.5\% & 0.6\% \\
Llama-3-8B, Chatbot Arena, $m{=}10$      & 9.9\% & 6.5\% & 3.7\% & 1.7\% & 0.7\% \\
Qwen2.5-72B, TL;DR, $m{=}10$             & 8.4\% & 5.2\% & 2.9\% & 1.3\% & 0.5\% \\
GPT-OSS-120B, TL;DR, $m{=}10$            & 7.1\% & 4.3\% & 2.2\% & 0.9\% & 0.3\% \\
GPT-OSS-120B, Chatbot Arena, $m{=}10$    & 8.2\% & 5.0\% & 2.7\% & 1.1\% & 0.4\% \\
GPT-OSS-120B, HH-RLHF, $m{=}10$          & 6.5\% & 3.9\% & 1.9\% & 0.8\% & 0.3\% \\
\bottomrule
\end{tabular}%
}
\caption{\textbf{Direct validation of the cross-stage monotonicity condition.}
Conditional shortlist miss probability (\%) as a function of the Stage~II margin threshold.}
\label{tab:cross-stage-monotonicity}
\end{table*}

For every configuration shown, the conditional miss probability decreases monotonically with $\lambda$.
At $\lambda{=}0$, it remains below the marginal failure level $\alpha{=}0.10$; at $\lambda{=}0.8$, it falls below 1\%.

\paragraph{Interpreting GSR.}
For each split, empirical agreement on accepted test instances estimates the inner population quantity
$\mathbb{P}_{(x,y)}(\hat y=y\mid\Delta_{\mathcal{S}}(x)\geq\lambda^*)$.
Across the 1{,}000 random calibration splits, GSR estimates the frequency with which this empirical agreement reaches the target $(1-\varepsilon)(1-\alpha)$, corresponding to the outer confidence requirement $1-\delta$.
Because each test set is finite, GSR should be interpreted as an empirical validation of the nested guarantee rather than an exact observation of the population probability.

\subsection{Variance-Normalized Single-Stage Baseline}

To test whether confidence-level normalization alone resolves probability dilution, we augment the direct $m$-to-$1$ adaptation of \citet{jung2024trust} with variance normalization.
\Cref{tab:variance-normalized-baseline} compares this strengthened baseline with the strongest single-stage scoring baseline and our two-stage method under the same representative setting.

\begin{table}[t]
\centering
\small
\resizebox{\columnwidth}{!}{%
\begin{tabular}{lcc}
\toprule
Method & GSR & Cov. \\
\midrule
Jung et al.\ ($m{>}2$ adapted) + variance normalization & 74.2\% & 53\% \\
Single-Stage (KL margin)                                & 76.4\% & 56\% \\
Two-Stage (KL margin)                                  & \textbf{96.8\%} & \textbf{70\%} \\
\bottomrule
\end{tabular}%
}
\caption{\textbf{Strengthened single-stage comparison} on TL;DR with GPT-OSS-120B
($m{=}10$, target agreement $=0.81$), averaged over 1{,}000 splits.}
\label{tab:variance-normalized-baseline}
\end{table}

In this setting, variance normalization does not close the gap: the two-stage method improves both GSR and coverage over the strengthened single-stage alternatives.

\subsection{Component-Wise Ablation of KL Margin}

We ablate the five components of the KL-margin score in \eqref{eq:kl-score}, removing one component at a time while retaining the other four.

\begin{table*}[!t]
\centering
\scriptsize
\resizebox{\textwidth}{!}{%
\begin{tabular}{lccccc cc}
\toprule
Variant & Renorm. & Exp. & Cond.\ Avg. & Var.\ Norm. & Scale & GSR & Cov. \\
\midrule
Full KL margin             & \checkmark & \checkmark & \checkmark & \checkmark & \checkmark & \textbf{96.8\%} & \textbf{70\%} \\
$-$ Renormalization        & -- & \checkmark & \checkmark & \checkmark & \checkmark & 83.5\% & 57\% \\
$-$ Variance normalization & \checkmark & \checkmark & \checkmark & -- & \checkmark & 89.2\% & 61\% \\
$-$ Scale factor           & \checkmark & \checkmark & \checkmark & \checkmark & -- & 91.4\% & 64\% \\
$-$ Conditional averaging  & \checkmark & \checkmark & -- & \checkmark & \checkmark & 92.6\% & 65\% \\
$-$ Exponential transform  & \checkmark & -- & \checkmark & \checkmark & \checkmark & 95.1\% & 68\% \\
\bottomrule
\end{tabular}%
}
\caption{\textbf{Component-wise KL-margin ablation} under the two-stage framework on TL;DR with GPT-OSS-120B
($m{=}10$, target agreement $=0.81$), averaged over 1{,}000 splits.}
\label{tab:kl-component-ablation}
\end{table*}

In this representative setting, removing any component reduces both GSR and coverage.
Shortlist renormalization and variance normalization produce the largest drops, while the remaining components provide smaller complementary gains.

\section{Full Per-Model and Per-Dataset Results}
\label{app:full-results}

This appendix supplements the main experiments (\Cref{sec:exp-motivation,sec:exp-main}) with results for additional judge models not shown in the main tables.
\Cref{tab:main} reports three representative judges from distinct families (Llama-3-8B, Qwen2.5-72B, GPT-OSS-120B);
here we present the same evaluations for three further judges---Mistral-7B, Qwen2.5-32B, and Llama-3-70B---covering the remaining model families and intermediate parameter scales.
Results for the other three judges (Qwen2.5-7B, DeepSeek-16B, DeepSeek-67B) follow the same trends and are omitted for brevity.

\paragraph{Ranking loss.}
\Cref{tab:ranking-loss-app} extends the monotonicity analysis of \Cref{tab:ranking-loss} to the additional judges.
The same pattern holds: single-stage ranking loss grows substantially with $m$, while both Stage~I and Stage~II maintain near-zero ranking loss across all settings.

\paragraph{Two-stage vs.\ single-stage.}
\Cref{tab:main-app} extends the main comparison of \Cref{tab:main} to the additional judges.
All trends are consistent: every two-stage variant exceeds the $1{-}\delta=90\%$ GSR target, while all single-stage baselines fall short.
Two-stage coverage improvements over the matched single-stage baseline range from 8--14 percentage points across all judges and datasets.
Beyond the tabular results, \Cref{fig:guarantee-chatbot,fig:guarantee-hh,fig:guarantee-gpt} extend the guarantee-curve and $m$-effect analysis of \Cref{fig:combined-results}(a)--(b) to additional (dataset, model) combinations.
The same qualitative conclusions hold: the two-stage empirical agreement consistently tracks above the diagonal across the full 70--95\% target range, and two-stage GSR remains above the $1{-}\delta=90\%$ threshold as $m$ grows from 5 to 20, whereas single-stage GSR degrades sharply.

\begin{figure}[h]
    \centering
    \includegraphics[width=\columnwidth]{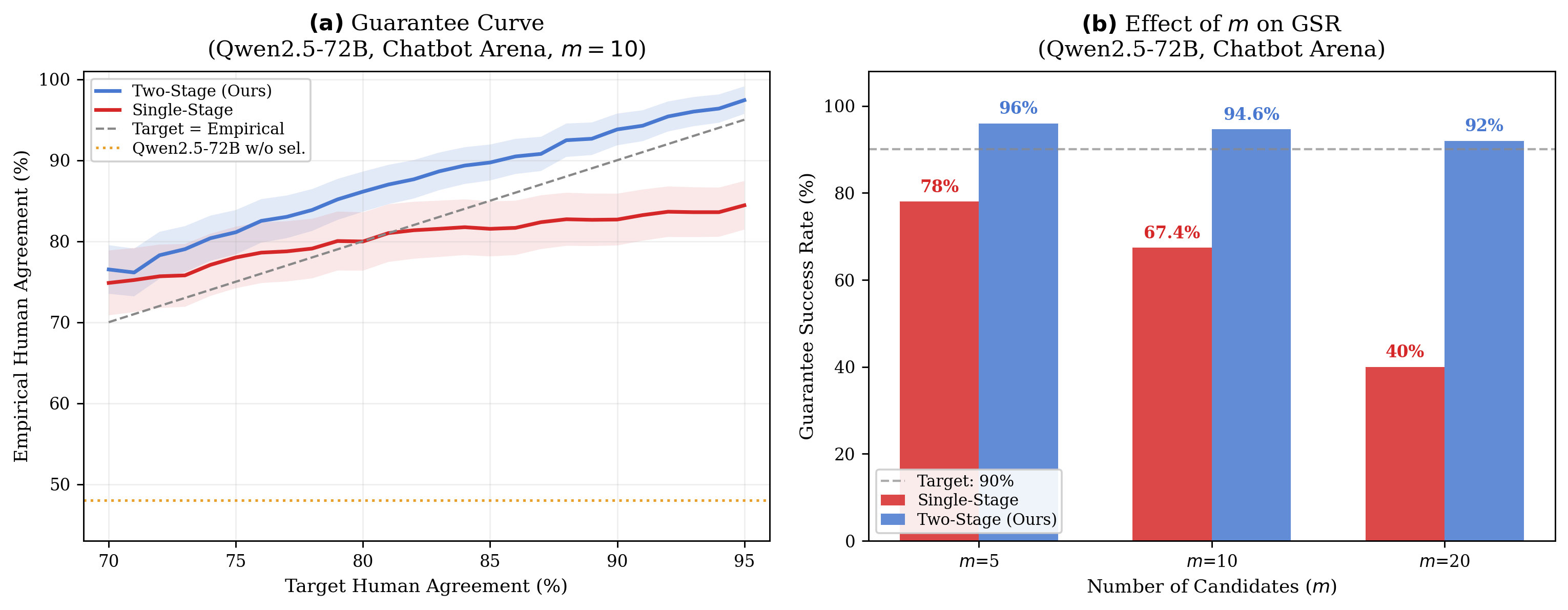}
    \vspace{-6mm}
    \caption{\textbf{Guarantee curve and $m$-effect on Chatbot Arena} (Qwen2.5-72B).
    Same layout as \Cref{fig:combined-results}(a)--(b).}
    \label{fig:guarantee-chatbot}
    \vspace{-3mm}
\end{figure}

\begin{figure}[h]
    \centering
    \includegraphics[width=\columnwidth]{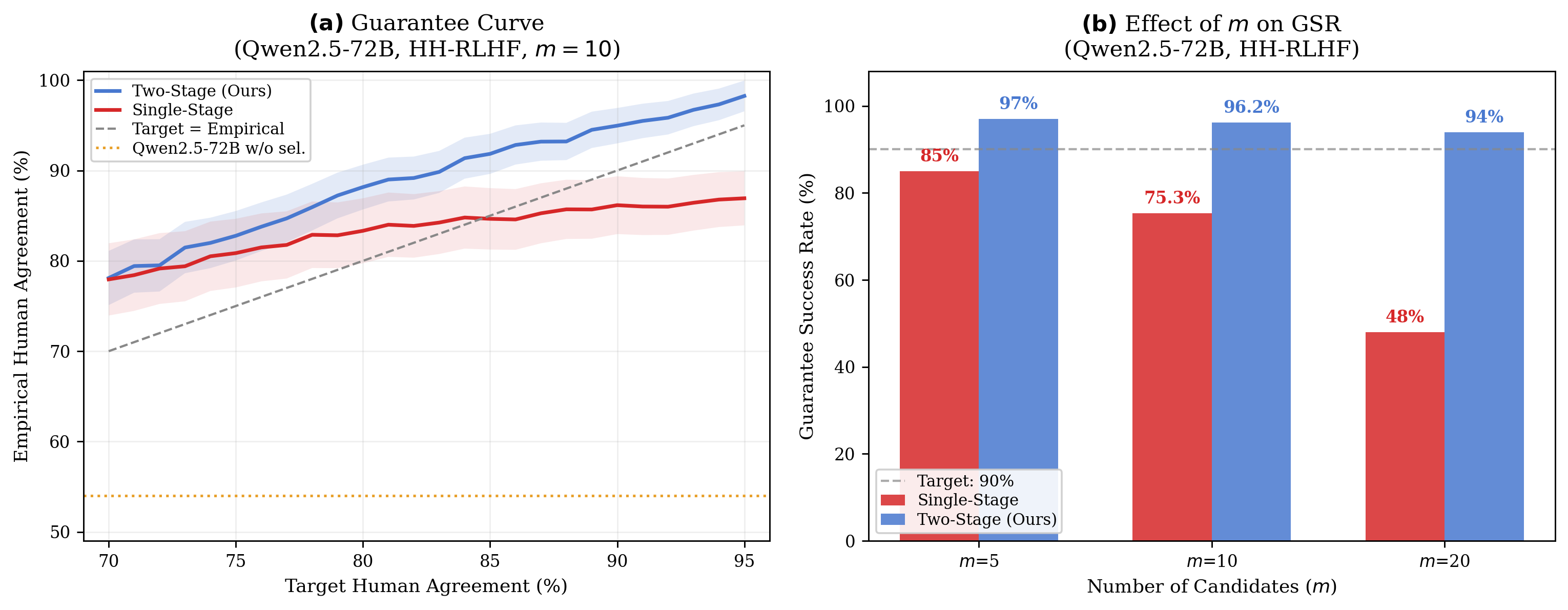}
    \vspace{-6mm}
    \caption{\textbf{Guarantee curve and $m$-effect on HH-RLHF} (Qwen2.5-72B).
    Same layout as \Cref{fig:combined-results}(a)--(b).}
    \label{fig:guarantee-hh}
    \vspace{-3mm}
\end{figure}

\begin{figure}[h]
    \centering
    \includegraphics[width=\columnwidth]{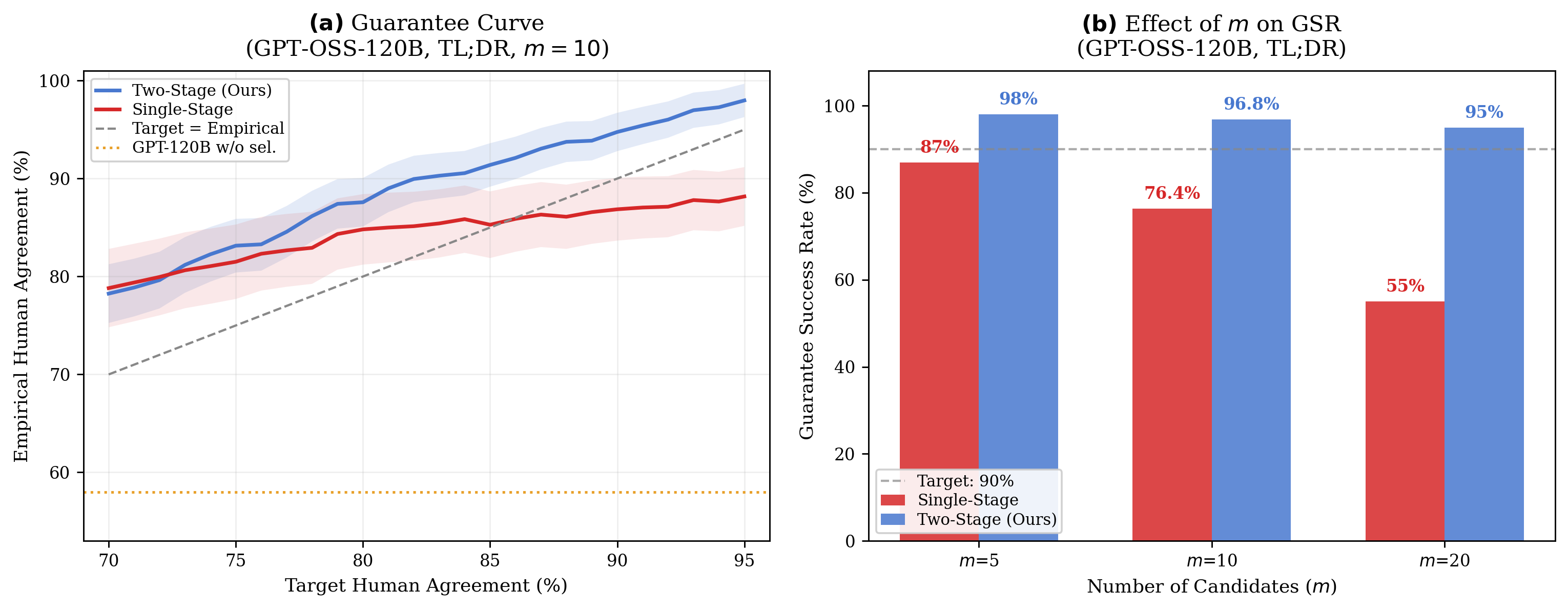}
    \vspace{-6mm}
    \caption{\textbf{Guarantee curve and $m$-effect on TL;DR} (GPT-OSS-120B).
    Same layout as \Cref{fig:combined-results}(a)--(b).}
    \label{fig:guarantee-gpt}
    \vspace{-3mm}
\end{figure}

\paragraph{Cascaded architecture on additional datasets.}
\Cref{tab:cascade} in the main text reports cascade results on TL;DR.
\Cref{tab:cascade-app} extends the evaluation to the remaining three datasets (Chatbot Arena, HH-RLHF, AlpacaEval) using representative cascade configurations.
The conclusions are fully consistent: all two-stage cascades exceed the $1{-}\delta=90\%$ GSR threshold, while all single-stage cascades fall short.
Cross-dataset variation is modest; three-tier two-stage cascades cover 73--81\% of instances and achieve 92--94\% GSR uniformly.
The tier-composition percentages are also stable across datasets: combining \Cref{tab:cascade} and \Cref{tab:cascade-app}, the weakest tier handles 32--39\% and the strongest tier handles 26--29\% of covered instances in three-tier configurations, confirming that the computational savings observed on TL;DR generalize broadly.
\Cref{fig:cascade-chatbot} further visualizes the cascaded guarantee curve and evaluator composition on Chatbot Arena, showing that the empirical agreement tracks the target throughout the 70--95\% range and that weaker tiers handle the majority of instances at lenient targets.

\begin{figure}[!b]
    \centering
    \includegraphics[width=\columnwidth]{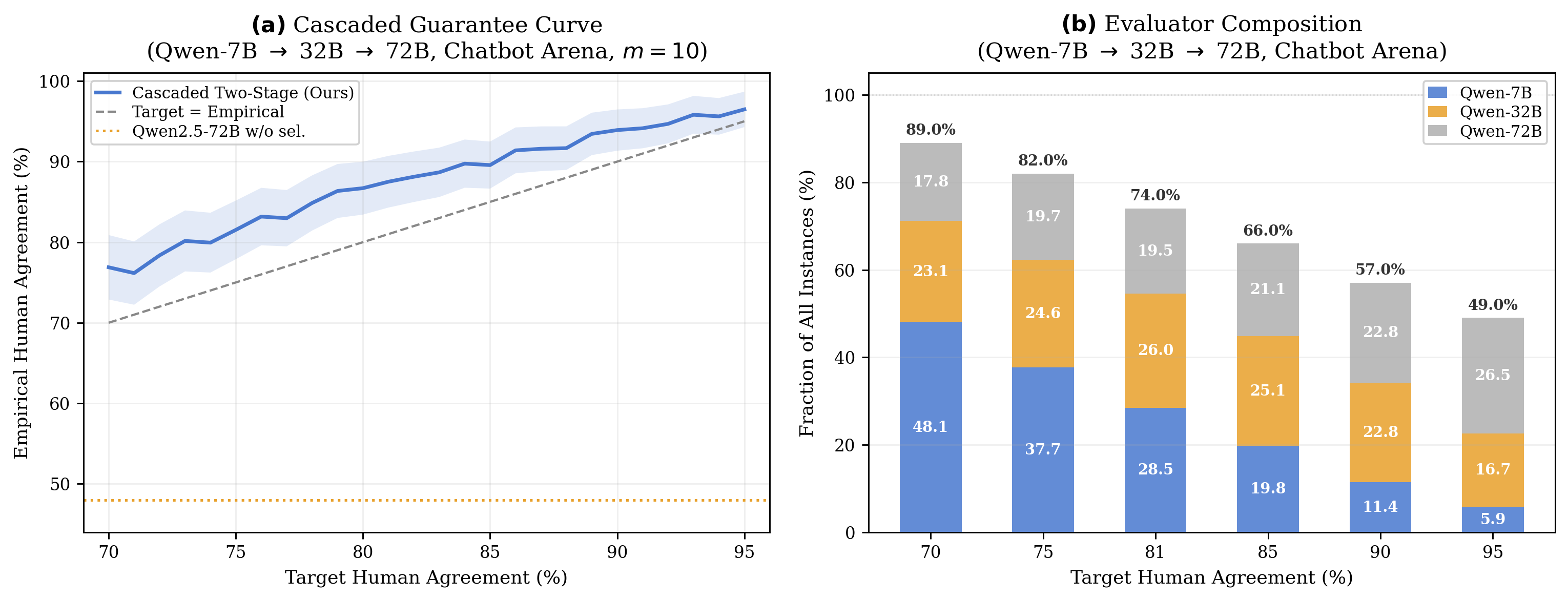}
    \vspace{-6mm}
    \caption{\textbf{Cascaded evaluation on Chatbot Arena} (Qwen-7B $\to$ 32B $\to$ 72B, $m{=}10$).
    Same layout as \Cref{fig:combined-results}(c)--(d).}
    \label{fig:cascade-chatbot}
\end{figure}

\end{document}